\documentclass{article}

\PassOptionsToPackage{table}{xcolor}
\usepackage{iclr2027_conference,times}
\iclrfinalcopy

\usepackage[utf8]{inputenc}
\usepackage[T1]{fontenc}
\usepackage[bookmarks=false]{hyperref}
\usepackage{url}
\usepackage{booktabs}
\usepackage{amsfonts}
\usepackage{amsmath}
\usepackage{amssymb}
\usepackage{nicefrac}
\usepackage{microtype}
\usepackage{xcolor}
\definecolor{citeblue}{RGB}{0, 148, 0}
\definecolor{citebluee}{RGB}{0, 148, 148}  
\hypersetup{
    colorlinks=true,
    linkcolor=red,     
    citecolor=citeblue, 
    urlcolor=citebluee,  
}
\usepackage{graphicx}
\usepackage{algorithm}
\usepackage{algorithmic}
\usepackage{multirow}
\usepackage{subcaption}
\usepackage{amsthm}
\usepackage{makecell}
\usepackage{wrapfig}
\usepackage{placeins}
\usepackage{tikz}
\usetikzlibrary{arrows.meta,positioning}
\usepackage{utfsym}
\newcommand{\heart}{$\;\!$\usym{2665}}

\newtheorem{definition}{Definition}
\newtheorem{theorem}{Theorem}
\newtheorem{proposition}{Proposition}

\definecolor{l1color}{RGB}{156,179,210}
\definecolor{l2color}{RGB}{228,170,122}
\definecolor{l3color}{RGB}{172,120,132}
\definecolor{l4color}{RGB}{145,185,175}
\definecolor{l5color}{RGB}{210,190,140}

\title{Self-Evolving Spatial Reasoning in Vision Language Models via Geometric Logic Consistency}

\author{
Junming Liu$^{1,2}$,
Yuqi Li$^{3}$,
Yifei Sun$^{2}$,
Maonan Wang$^{4}$,
Yiming Cheng$^{5}$,
Rui Qian$^{6}$, \\[1mm]
\textbf{
Piotr Koniusz$^{7,8}$\thanks{Corresponding authors}\ \ ,
Yirong Chen$^{2}$\footnotemark[2]\ \ ,
and Ding Wang$^{2}$}\footnotemark[2]\ \ \\[2mm]
$^{1}$The University of Melbourne, \\
$^{2}$Shanghai Artificial Intelligence Laboratory, \\
$^{3}$The City University of New York, \\
$^{4}$The Chinese University of Hong Kong, \\
$^{5}$Tsinghua University, \\
$^{6}$Fudan University, \\
$^{7}$University of New South Wales, \\
$^{8}$Data61$\!${\color{red}\heart}CSIRO, \\
}

\begin{document}

\setcounter{footnote}{1}
\maketitle
\lhead{Preprint.}

\begin{abstract}
Vision-Language Models (VLMs) have made striking progress, yet their spatial reasoning remains fragile. Models that answer an original input correctly can still fail under valid transformations with predictable answer mappings, revealing a gap between instance-level correctness and robust spatial reasoning. To address this, we propose \textbf{S}patial \textbf{A}lignment via \textbf{G}eometric \textbf{E}volution (SAGE), a self-evolving framework that improves robust spatial reasoning through geometric and linguistic duality operations. SAGE incorporates duality consistency into GRPO training, encouraging models to produce coherent answers across original and transformed inputs. SAGE co-evolves duality generation and solution, allowing the model to continually expose and address its own reasoning weaknesses. A dynamic operation pool identifies challenging operations and retires mastered ones, keeping training focused on informative duality signals. SAGE is model-agnostic, data-efficient compared to prior post-training methods, and can be applied as a lightweight adaptation stage to any existing VLM. Experiments on video and spatial reasoning benchmarks demonstrate consistent improvements over strong baselines and enhanced generalization to unseen data.
\end{abstract}

\section{Introduction}
\label{sec:intro}

Vision-Language Models (VLMs) \cite{Singh_2025_GPT-5, Bai_2025_Qwen3-VL} have advanced rapidly in recent years, evolving from basic perception toward comprehensive world understanding and multimodal reasoning. Despite this progress, their representation of spatial relationships and physical scenes remains limited~\cite{Cheng_2024_SpatialRGPT}. When asked about object positions, motion directions, relative distances, or depth orderings, VLMs often produce confident but incorrect answers, exhibiting spatial hallucinations that undermine their reliability in embodied navigation \cite{Zheng_2024_Embodied} and autonomous driving \cite{Zhou_2026_Autonomous}. Since these applications require models to reason consistently about how objects are arranged and how scenes change under viewpoint or geometric transformations, improving spatial reasoning has become a central challenge for multimodal intelligence.

Recent methods have applied Reinforcement Learning (RL) to improve VLM reasoning, including spatial and temporal reasoning tasks \cite{Zhao_2025_Embodied-R, Pan_2026_MetaSpatial, Li_2026_SpatialLadder}. These methods design reward signals to encourage desirable behaviors such as correct answers, improved temporal modeling \cite{Feng_2025_Video-R1}, or more precise geometric reasoning \cite{Wu_2025_VILASR}. However, existing objectives do not directly enforce consistency across paired formulations of the same spatial relation \cite{Li_2026_STAR-R1, Zhang_2026_SpinBench}. For example, after a horizontal flip, the correct answer to a left/right question should change predictably; if a model answers the original input correctly but fails on its flipped counterpart, it exposes a gap between answer correctness and robust spatial reasoning \cite{Kim_2026_SCOPE}. As shown in Figure~\ref{fig:intro}, such transformed questions are often intuitive to humans but remain challenging for VLMs. Moreover, although related forms of transformations have been considered in prior work~\cite{Fang_2025_ViSS-R1, Zhou_2026_PoT}, they are typically treated as predefined augmentations or task-specific perturbations rather than formalized as duality operations that can adapt to the model's evolving weaknesses.

\begin{figure}[t]
\centering
\includegraphics[width=\linewidth]{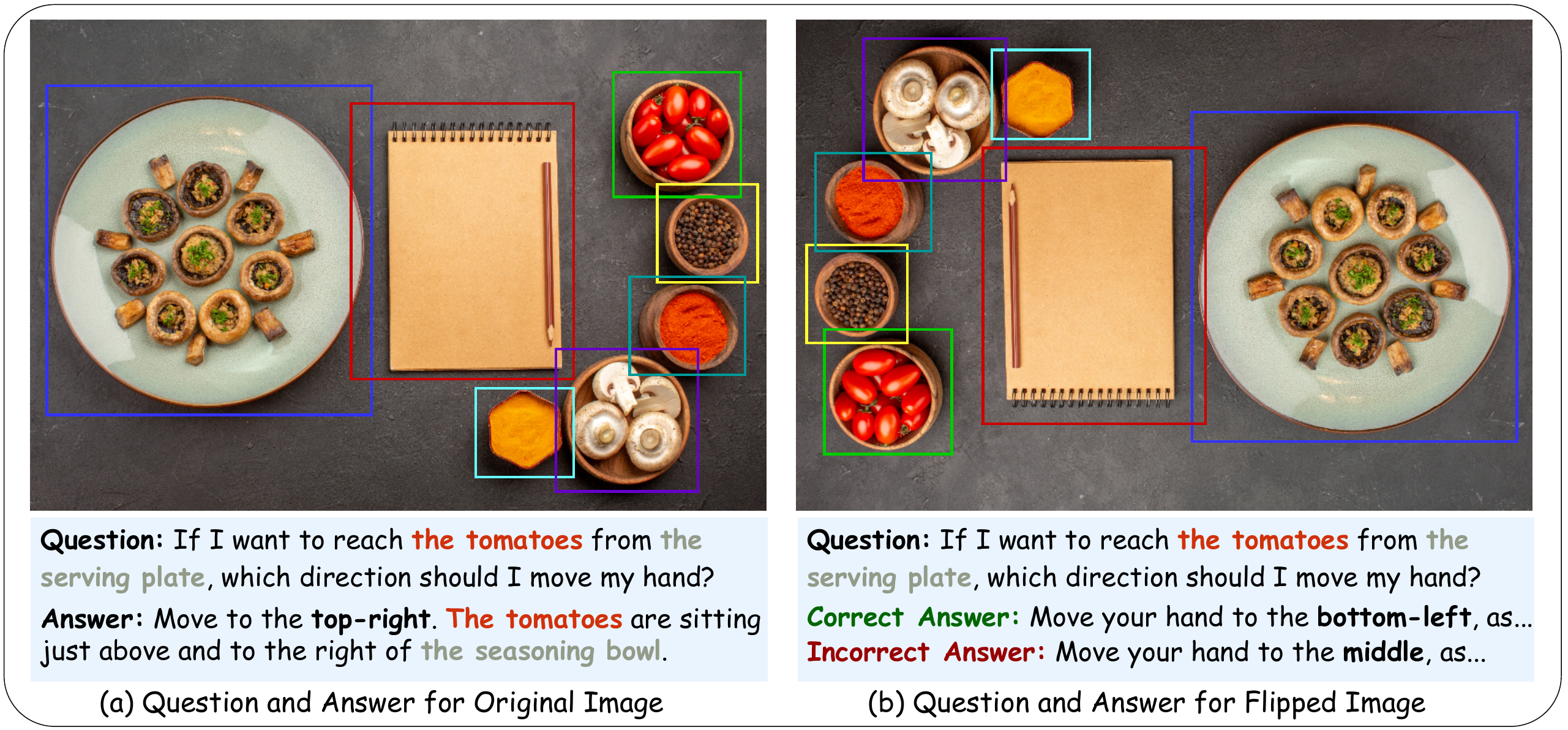}
\caption{Analysis of spatial consistency in VLMs under geometric transformations.}
\label{fig:intro}
\end{figure}

To address this challenge, we propose \textbf{S}patial \textbf{A}lignment via \textbf{G}eometric \textbf{E}volution (SAGE), a self-evolving post-training framework that improves VLM spatial reasoning by enforcing duality consistency. SAGE formalizes paired transformations as \emph{duality operations}, including geometric changes to the visual input, such as spatial reflections, and structural changes to the query, such as option permutation. Each duality operation specifies how the input is transformed and how the corresponding answer should change, allowing the model to be trained toward answer-level equivariance. Given an original sample and its dual counterpart, SAGE encourages the model to produce logically consistent answers under the known transformation mapping. We integrate this consistency signal as an auxiliary reward within GRPO training, providing supervision beyond standard answer accuracy.
Crucially, SAGE does not rely on a fixed set of transformations throughout training. Instead, it maintains a dynamic pool of duality operations that periodically probes the model for inconsistencies, identifies transformations associated with current reasoning failures, and promotes useful operations for active training. This self-evolving process allows the training distribution to adapt to the model's observed inconsistency patterns rather than relying on a manually fixed set of transformations. SAGE further uses a shared VLM for both proposing candidate duality operations and solving the resulting original--dual pairs, with deterministic validation used to verify the generated operations. This shared-parameter design avoids introducing additional trainable role-specific models while preserving the distinction between operation generation and solution.
We evaluate SAGE across a broad set of video understanding and spatial reasoning benchmarks to determine whether duality consistency improves generalization beyond the specific transformations used during training. Experiments on six video understanding and seven spatial reasoning benchmarks show consistent improvements on the majority of benchmarks. These results suggest that duality consistency provides a complementary post-training signal to task-specific rewards and helps expose reasoning failures that standard answer accuracy alone may overlook.

Our technical contributions are summarized as follows:
\begin{enumerate}
\item We introduce \emph{duality consistency} as a general training signal for VLM spatial reasoning. By formulating visual and query transformations together with their corresponding answer mappings, SAGE trains models toward answer-level equivariance and exposes inconsistencies that standard answer accuracy may overlook.

\item We propose SAGE, a self-evolving post-training framework that integrates duality consistency into GRPO. SAGE uses a shared model to generate and solve duality operations, while maintaining a dynamic operation pool that adaptively prioritizes transformations associated with the model's current reasoning vulnerabilities.

\item Across six video understanding benchmarks and seven spatial reasoning benchmarks, SAGE improves duality consistency and achieves competitive or superior performance using substantially less training data than prior GRPO-based methods, showing that consistency-based rewards complement task-specific post-training.
\end{enumerate}

\section{Method}
\label{sec:method}

Robust spatial reasoning requires both correct answers and coherent predictions across inputs linked by known answer transformations. SAGE implements this principle in three stages: (1) \emph{duality probing} detects inconsistencies on original--dual pairs; (2) a \emph{self-evolving operation pool} autonomously discovers and schedules transformations; and (3) \emph{consistency-augmented GRPO} converts these signals into training rewards. The overall workflow is illustrated in Figure~\ref{fig:framework}.

\subsection{Stage 1: Duality Probing}
\label{sec:duality_probing}

Contemporary VLMs achieve high accuracy on spatial reasoning benchmarks, yet correct answers alone do not establish that their predictions remain consistent under valid transformations. To formalize this, we introduce the notion of \emph{duality complement}. Let $\mathcal M_\theta$ denote a VLM parameterized by $\theta$ that, given a visual input $v\in\mathcal V$ and a query $q\in\mathcal Q$, produces a distribution over answers $\mathcal M_\theta(a\mid v,q)$.
Let $x=(v,q)\sim\mathcal D$ with $x\in\mathcal X=\mathcal V\times\mathcal Q$, and let $\mathcal A$ be the canonical answer space. We denote the decoded answer and ground truth by $\hat a_\theta(x)$ and $a^*(x)\in\mathcal A$, respectively, with $\bot\notin\mathcal A$ indicating extraction failure and counted as incorrect.

\begin{definition}[Duality operation]
\label{def:duality_op}
A duality operation is a triple $\mathcal{T}=(T_{\mathcal X},T_{\mathcal A},\mathcal S)$, where
$T_{\mathcal X}:\mathcal X\to\mathcal X$ transforms the input,
$T_{\mathcal A}:\mathcal A\to\mathcal A$ is the induced bijective answer transformation,
and $\mathcal S\subseteq\mathcal X$ is the applicability domain, such that
\begin{equation}
    \forall x\in\mathcal S:\qquad a^*\bigl(T_{\mathcal X}(x)\bigr)=T_{\mathcal A}\bigl(a^*(x)\bigr).
    \label{eq:duality_axiom}
\end{equation}
\end{definition}

\begin{figure}[t]
\centering
\includegraphics[width=\linewidth]{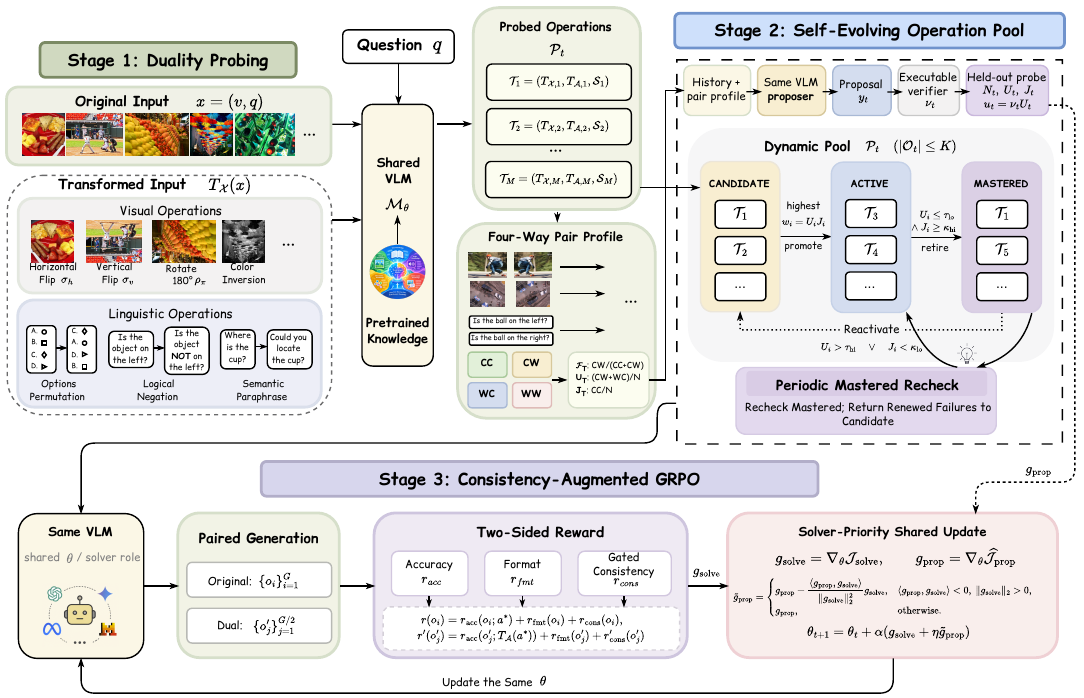}
\caption{Overview of SAGE. Stage~1 probes $\mathcal{M}_\theta$ with duality operations to detect reasoning inconsistency. Stage~2 manages the operation pool through a promote--retire--recheck lifecycle. Stage~3 integrates duality consistency as an auxiliary reward within GRPO training.}
\label{fig:framework}
\end{figure}

Equation~\eqref{eq:duality_axiom} defines answer-level equivariance through a deterministic relation between dual answers, with invariance as the special case $T_{\mathcal A}=\mathrm{id}$.

\begin{definition}[Duality consistency]
For $\Pr_{x\sim\mathcal D}[x\in\mathcal S]>0$, let
$\mathcal D_{\mathcal S}=\mathcal D(\cdot\mid x\in\mathcal S)$. The duality consistency of $\mathcal M_\theta$ with respect to $\mathcal T$ is
\begin{equation}
    \mathcal C_{\mathcal T}(\theta)
    =
    \mathbb E_{x\sim\mathcal D_{\mathcal S}}
    \left[
    \mathbf 1\left[
    T_{\mathcal A}(\hat a_\theta(x))
    =
    \hat a_\theta(T_{\mathcal X}(x))
    \right]
    \right],
    \label{eq:consistency_def}
\end{equation}
where the indicator is defined as zero if either answer extraction fails.
\end{definition}

Since consistent answers can both be wrong, we define the hidden dual failure rate conditioned on original correctness, requiring $\Pr_{x\sim\mathcal D_{\mathcal S}}[\hat a_\theta(x)=a^*(x)]>0$:
\begin{equation}
    \mathcal F_{\mathcal T}(\theta)
    =
    \Pr_{x\sim\mathcal D_{\mathcal S}}
    \left[
    \hat a_\theta(T_{\mathcal X}(x))
    \neq
    T_{\mathcal A}(a^*(x))
    \,\middle|\,
    \hat a_\theta(x)=a^*(x)
    \right].
    \label{eq:hidden_failure}
\end{equation}

For a finer analysis, we partition applicable pairs into $\{\mathrm{CC},\mathrm{CW},\mathrm{WC},\mathrm{WW}\}$, where C/W denote correct/wrong predictions on the original and dual inputs, respectively.
Let $N_{\mathcal T}^{z}$ denote the count of case $z$, and let $N_{\mathcal T}$ be the total number of applicable pairs.
We report the one-sided failure rate and joint-correctness rate:
\begin{equation}
    U_{\mathcal T}
    =
    \frac{N_{\mathcal T}^{\rm CW}+N_{\mathcal T}^{\rm WC}}
    {N_{\mathcal T}},
    \qquad
    J_{\mathcal T}
    =
    \frac{N_{\mathcal T}^{\rm CC}}
    {N_{\mathcal T}}.
    \label{eq:pair_scores}
\end{equation}

\paragraph{Visual duality.}
These operations transform $v$ while retaining $q$. Initial operations include \emph{horizontal reflection} $\sigma_h$ and \emph{vertical reflection} $\sigma_v$, whose answer maps swap left$\leftrightarrow$right and top$\leftrightarrow$bottom, respectively; \emph{$180^\circ$ rotation} $\rho_\pi$, which composes both mappings; and \emph{appearance perturbations} (color inversion, grayscale), with $T_{\mathcal A}=\mathrm{id}$ only when the answer is unaffected.
Flips exclude text-reading questions and samples for which the directional mapping is insufficient. Input transformations need not be invertible, but must satisfy Eq.~\eqref{eq:duality_axiom} on their applicability domain.

\paragraph{Linguistic duality.}
These operations retain $v$ and transform $q$ through \emph{option permutation}, which reorders $C$ choices by $\pi\in S_C$ with $T_{\mathcal A}(a)=\pi^{-1}(a)$; \emph{logical negation}, which complements an unambiguous binary answer and applies to multiple-choice questions only when the transformed answer is unique; or \emph{semantic paraphrase}, where $T_{\mathcal X}(v,q)=(v,q')$ preserves meaning and $T_{\mathcal A}=\mathrm{id}$.
Additional visual or linguistic operations can be discovered subject to Eq.~\eqref{eq:duality_axiom} on their applicability domains.

\paragraph{Composition.}
Valid $\mathcal T_1$ and $\mathcal T_2$ compose through $T_{\mathcal X,1}\circ T_{\mathcal X,2}$ and $T_{\mathcal A,1}\circ T_{\mathcal A,2}$, with applicability domain $\mathcal S_2\cap T_{\mathcal X,2}^{-1}(\mathcal S_1)$.

\paragraph{Efficient batch probing.}
We process $N$ sampled applicable pairs in one batched inference through $\mathcal M_\theta$, recording all four outcomes, including original-side errors.

\subsection{Stage 2: Self-Evolving Operation Pool}
\label{sec:self_evolving}

Not all duality operations are equally informative throughout training. Operations that the model already solves correctly on both sides provide limited additional supervision, whereas operations yielding consistently wrong predictions provide little useful consistency signal. We therefore maintain a dynamic pool $\mathcal P_t$ and allow the same trainable VLM to discover, evaluate, and solve new operations online.

\paragraph{Proposer-format initialization.}
A short SFT stage teaches the model to serialize the $(T_{\mathcal X},T_{\mathcal A},\mathcal S)$ contract from context--proposal examples $(h,y)$:
\begin{equation}
\mathcal L_{\rm prop\text{-}SFT}(\theta)
=-\mathbb E_{(h,y)}\sum_{\ell=1}^{|y|}
\log\mathcal M_\theta(y_\ell\mid h,y_{<\ell}).
\label{eq:proposer_sft}
\end{equation}
The resulting checkpoint $\mathcal M_{\theta_0}$ serves as both proposer and solver during online training; no separate proposer or learned judge is introduced.

\paragraph{Candidate discovery.}
Every $E_p$ steps, the current model receives a proposer context $h_t$ containing compact operation history and pair profiles, and samples a serialized proposal $y_t=(y_{t,1},\ldots,y_{t,L})$. The proposal is parsed into a candidate
$\widetilde{\mathcal T}_t=(\widetilde T_{\mathcal X,t},\widetilde T_{\mathcal A,t},\widetilde{\mathcal S}_t)$ containing executable input- and answer-transformation rules, an applicability predicate, and a rationale. Applying the rule to labeled source items yields original--dual pairs with mechanically derived answers. The initial validated operations provide a starting set, which the proposer continually expands or refines.

\paragraph{Validation and admission.}
A parameter-free verifier returns $\nu_t\in\{0,1\}$ through executable checks, not a proof of validity. Passing candidates receive probe scores $(N_t,U_t,J_t)$ on a buffer reserved from solver rollouts for training feedback, not independent evaluation. Proposal utility is
\begin{equation}
u_t=\nu_t U_t,
\label{eq:proposal_utility}
\end{equation}
with $U_t=0$ for failed validation or empty probes. A proposal enters $\mathcal P_t$ as \textsc{Candidate} only if $\nu_t=1$, $N_t\geq N_{\min}>0$, and $U_t\geq\tau_{\rm add}\in(0,1]$.

\paragraph{Proposal update.}
Because $y_t$ is discrete, we optimize its likelihood with the sample policy-gradient surrogate
\begin{equation}
\widehat{\mathcal J}_{{\rm prop},t}(\theta)
=\operatorname{sg}(u_t-b_t)\sum_{\ell=1}^{L}
\log\mathcal M_\theta(y_{t,\ell}\mid h_t,y_{t,<\ell}),
\label{eq:proposal_objective}
\end{equation}
where $\operatorname{sg}$ stops gradients through the utility and baseline. Starting from $b_0=0$, the baseline is updated after each proposal as
$b_{t+1}=\rho b_t+(1-\rho)u_t$, $0\leq\rho<1$; the surrogate contributes a positive update when $u_t>b_t$.

\paragraph{States and transitions.}
Each operation $\mathcal T_i$ stores exponentially smoothed scores $U_i$ and $J_i$, initialized from its admission probe.
Profiles are updated only with at least $N_{\min}$ applicable pairs; unsupported estimates remain unchanged. Each operation has state $s_i\in\{\textsc{Candidate},\textsc{Active},\textsc{Mastered}\}$, with active subset $\mathcal O_t\subseteq\mathcal P_t$ satisfying $|\mathcal O_t|\leq K$. We assign each candidate a priority
\begin{equation}
w_i=U_iJ_i,
\label{eq:operation_priority}
\end{equation}
which favors operations that expose one-sided failures while retaining substantial joint correctness. Every $E$ solver steps, supported probes trigger state transitions, while candidates fill vacant active slots:
\begin{align}
s_i:\textsc{Active}
&\xrightarrow{U_i\leq\tau_{\mathrm{lo}}\ \land\ J_i\geq\kappa_{\mathrm{hi}}}
\textsc{Mastered}, \label{eq:retire}\\
s_i:\textsc{Mastered}
&\xrightarrow{U_i>\tau_{\mathrm{hi}}\ \lor\ J_i<\kappa_{\mathrm{lo}}}
\textsc{Candidate}, \label{eq:forget}\\
s_i:\textsc{Candidate}
&\xrightarrow{\text{highest }w_i,\ |\mathcal O_t|<K}
\textsc{Active}. \label{eq:promote}
\end{align}
Here $\tau_{\mathrm{lo}}$, $\tau_{\mathrm{hi}}$, $\kappa_{\mathrm{lo}}$, and $\kappa_{\mathrm{hi}}$ control retirement and reactivation.

\subsection{Stage 3: Consistency-Augmented GRPO}
\label{sec:duality_reward}

We integrate duality consistency into the GRPO framework~\citep{Shao_2024_GRPO}. At each solver step, given $x=(v,q)$ and its label $a^*=a^*(x)$:

\paragraph{Generation.} The model generates $G$ completions $\{o_i\}_{i=1}^G\sim\mathcal M_{\theta_t}(\cdot\mid x)$, where $G\geq4$ is even.
An active operation $\mathcal T\in\mathcal O_t$ with $x\in\mathcal S_{\mathcal T}$ is sampled with probability proportional to its priority; if all applicable priorities are zero, sampling is uniform. The same model generates $G/2$ dual completions $\{o'_j\}_{j=1}^{G/2}\sim\mathcal M_{\theta_t}(\cdot\mid T_{\mathcal X}(x))$. If no active operation applies, the step uses original-only GRPO.

\paragraph{Reward.} Rewards combine task accuracy, format adherence, and duality consistency, supervising original and dual completions against $a^*$ and the mechanically derived $T_{\mathcal A}(a^*)$, respectively:
\begin{equation}
\begin{aligned}
    r(o_i) &= r_\text{acc}(o_i;a^*)+r_\text{fmt}(o_i)+r_\text{cons}(o_i),\\
    r'(o'_j) &= r_\text{acc}(o'_j;T_{\mathcal A}(a^*))+r_\text{fmt}(o'_j)+r'_\text{cons}(o'_j).
\end{aligned}
    \label{eq:total_reward}
\end{equation}
The consistency bonus is gated by task correctness. Let
\begin{equation}
 c_{ij}=\mathbf 1\!\left[T_{\mathcal A}(\operatorname{ans}(o_i))=\operatorname{ans}(o'_j)\right], \quad
 g_{ij}=\mathbf 1\!\left[\operatorname{ans}(o_i)=a^*\ \lor
 \operatorname{ans}(o'_j)=T_{\mathcal A}(a^*)\right].
\end{equation}
The symmetric consistency bonuses are
\begin{equation}
    r_\text{cons}(o_i)=\lambda\frac{2}{G}\sum_{j=1}^{G/2}g_{ij}c_{ij}, \quad 
    r'_\text{cons}(o'_j)=\lambda\frac{1}{G}\sum_{i=1}^{G}g_{ij}c_{ij}.
    \label{eq:consistency_reward}
\end{equation}
where $\operatorname{ans}(\cdot)$ extracts a canonical answer and $\lambda\geq0$ controls the weight. Extraction failure sets $c_{ij}=g_{ij}=0$. For a valid operation, bijectivity implies $g_{ij}c_{ij}=1$ if and only if both answers are correct; consistent but wrong pairs receive no bonus.

\paragraph{Policy update.} Following GRPO, we compute group-relative advantages separately for the original and dual prompts:
\begin{equation}
 \hat A_i=\frac{r(o_i)-\bar r}{\sigma_r+\epsilon},\qquad
 \hat A'_j=\frac{r'(o'_j)-\bar r'}{\sigma_{r'}+\epsilon},
 \label{eq:group_advantages}
\end{equation}
where each group supplies its own mean and standard deviation, with $\epsilon>0$. Holding rollouts and advantages fixed, we equally weight the two clipped GRPO policy surrogates and apply a single KL penalty:
\begin{equation}
 \mathcal J_{\rm solve}(\theta)
 =\tfrac12\mathcal J_{\rm GRPO}^{\rm orig}(\theta)
 +\tfrac12\mathcal J_{\rm GRPO}^{\rm dual}(\theta)
 -\beta
 D_{\rm KL}\!\left[
 \mathcal M_\theta(\cdot\mid x)
 \|
 \mathcal M_{\theta_{\rm ref}}(\cdot\mid x)
 \right],
 \label{eq:grpo_loss}
\end{equation}
where $\beta\geq0$ controls the KL penalty against a frozen reference model, applied to the original prompt and omitted when no reference is used.
All four correctness outcomes participate in task rewards; a group with identical rewards has zero advantages.

\paragraph{Solver-priority shared update.}
\label{sec:shared_update}
Because proposal generation and solving share $\theta$, their gradients may conflict. At $\theta=\theta_t$, we compute
\begin{equation}
 g_{\rm solve}=\nabla_\theta\mathcal J_{\rm solve},\qquad
 g_{\rm prop}=\nabla_\theta\widehat{\mathcal J}_{{\rm prop},t},
\end{equation}
and project only when the two gradients conflict:
\begin{equation}
 \widetilde g_{\rm prop}=
 \begin{cases}
 g_{\rm prop}-\dfrac{\langle g_{\rm prop},g_{\rm solve}\rangle}
 {\|g_{\rm solve}\|_2^2}\,g_{\rm solve},
 & \langle g_{\rm prop},g_{\rm solve}\rangle<0
   \ \text{and}\ \|g_{\rm solve}\|_2>0,\\[6pt]
 g_{\rm prop},&\text{otherwise}.
 \end{cases}
 \label{eq:projected_gradient}
\end{equation}
The shared update, with step size $\alpha>0$ and proposer weight $\eta\geq0$, is
\begin{equation}
 \theta_{t+1}=\theta_t+\alpha\bigl(g_{\rm solve}+\eta\widetilde g_{\rm prop}\bigr).
 \label{eq:joint_objective}
\end{equation}
The projection gives $\langle g_{\rm solve},\widetilde g_{\rm prop}\rangle\geq0$, mitigating gradient conflict with one trainable parameter and optimizer state. This is a first-order guarantee for the sampled solver surrogate under raw gradients, not for task performance or updates with momentum, preconditioning, or weight decay.

\section{Experiments}
\label{sec:experiments}

\subsection{Experimental Setup}
\label{main_setup}

\paragraph{Datasets.}
We evaluate SAGE on six video understanding benchmarks: MVBench \cite{Li_2024_MVBench}, TempCompass \cite{Liu_2024_TempCompass}, VideoMMMU \cite{Hu_2025_Video-MMMU}, Video-MME \cite{Fu_2025_Video-MME}, VSI-Bench \cite{Yang_2025_VSI-Bench}, and MMVU \cite{Zhao_2025_MMVU}; and seven spatial reasoning benchmarks: CV-Bench \cite{Tong_2024_CV-Bench}, SPAR-Bench \cite{Zhang_2026_SPAR-Bench}, ViewSpatial-Bench \cite{Li_2025_ViewSpatial-Bench}, MMSI-Bench \cite{Yang_2026_MMSI-Bench}, MindCube \cite{Yin_2025_MindCube}, OST-Bench \cite{Lin_2026_OST-Bench}, and OmniSpatial \cite{Jia_2026_OmniSpatial}.

\paragraph{Baselines.}
For video understanding, we compare against a broad range of state-of-the-art VLMs, including proprietary models such as GPT-4o \cite{Hurst_2024_GPT-4o}, Gemini-2.0-Flash \cite{Gemini_2024_Gemini1.5}, GPT-5, GPT-5.4 \cite{Singh_2025_GPT-5}, and Gemini-3-Pro-Preview, as well as open-source models including InternVL-2.5 \cite{Chen_2025_InternVL-2.5}, LLaMA-VID \cite{Li_2024_LLaMA-VID}, VideoLLaMA2 \cite{Cheng_2024_VideoLLaMA2}, LongVA-7B \cite{Zhang_2025_LongVA-7B}, VILA-1.5-40B \cite{Lin_2024_VILA}, Video-UTR-7B \cite{Yu_2025_Video-UTR}, LLaVA-OneVision-7B \cite{Li_2025_LLaVA-OneVision}, Kangaroo-8B \cite{Liu_2026_Kangaroo}, Video-R1-7B \cite{Feng_2025_Video-R1}, Video-RTS-7B \cite{Wang_2025_Video-RTS}, VideoRFT-7B \cite{Wang_2025_VideoRFT}, ViSS-R1-7B \cite{Fang_2025_ViSS-R1}, Video-RAISE-7B \cite{Xi_2026_VideoRAISE}, Video-KTR-7B \cite{Wang_2026_VideoKTR}, VideoSSR-8B \cite{He_2026_VideoSSR}, VideoAuto-R1-8B \cite{Liu_2026_VideoAuto-R1}, and OneThinker-8B \cite{Feng_2026_OneThinker}. For spatial reasoning, we compare with GPT-4o, InternVL-2.5-8B, LLaVA-OneVision-7B, SpaceR \cite{Ouyang_2025_SpaceR}, VILASR \cite{Wu_2026_VILASR}, Spatial-MLLM \cite{Wu_2025_Spatial-MLLM}, SpatialLadder-3B \cite{Li_2026_SpatialLadder}, VST-7B-RL \cite{Yang_2025_VST}, and SenseNova-SI \cite{Cai_2026_SenseNovaSI}.
Due to the computational cost of reproducing all baselines, we combine reported results with our own evaluations, following the settings specified in the original papers. Appendix~\ref{sec:baseline_sources} details the result sources and evaluation protocols.

\paragraph{Training and evaluation.}
We conduct experiments with Qwen2.5-VL-7B-Instruct~\cite{Bai_2025_Qwen2.5-VL}, Qwen3-VL-4B/8B-Instruct~\cite{Bai_2025_Qwen3-VL}. We perform proposer SFT on DeepSeek-V4-Pro-generated duality code, followed by GRPO on 5K Video-R1~\cite{Feng_2025_Video-R1} and 5K SpatialLadder~\cite{Li_2026_SpatialLadder} samples.
RL training uses 16 NVIDIA A100 GPUs with DeepSpeed ZeRO-3, a learning rate of $10^{-6}$, batch size 1 per device, and 8 original rollouts per prompt for one epoch. Unless otherwise specified, $\lambda=0.3$, $\beta=0.04$, $K=3$, $E=100$, $E_p=100$, $\eta=0.01$, and $p_f=0.2$. We use the same decoding and answer extraction procedure across methods and evaluate duality consistency on valid, applicable original--transformed pairs, reporting four correctness outcomes: correct--correct (CC), correct--wrong (CW), wrong--correct (WC), and wrong--wrong (WW). Detailed settings are provided in Appendix~\ref{app:training_config}.

\subsection{Main Results}
\label{main_results}

\paragraph{Video understanding.}
Table~\ref{tab:simplified_results} shows that SAGE improves all three backbones, with the largest gains on VideoMME (+15.7--16.3 points) and substantial gains on MVBench (+9.0--14.5). By encouraging correct answers across original and transformed inputs under the prescribed answer mapping, duality training couples augmentation with a reasoning constraint. This may reduce reliance on orientation-specific or superficial cues, supporting more robust grounding across diverse scenes and events. The smaller gains on VSI-Bench (+1.6--6.7) suggest that transformation robustness alone does not resolve spatial-memory and metric-estimation demands, while VideoMMMU and MMVU additionally require knowledge acquisition or domain expertise. These results are consistent with improved generalization through robust visual reasoning, with benefits depending on each task's underlying bottlenecks.

\begin{table}[t]
\centering
\caption{Performance comparison on video understanding benchmarks.}
\label{tab:simplified_results}
\small
\setlength{\tabcolsep}{2pt}
\begin{tabular*}{\textwidth}{@{}l@{\hspace{1em}\extracolsep{\fill}}*{6}{c@{\,}l}@{}}
\toprule
Model & \makebox[0pt][c]{VSI-Bench} & & \makebox[0pt][c]{VideoMMMU} & & \makebox[0pt][c]{MMVU (MC)} & & \makebox[0pt][c]{MVBench} & & \makebox[0pt][c]{TempCmps} & & \makebox[0pt][c]{VideoMME} & \\
\midrule

\rowcolor{l1color!30}
\multicolumn{13}{c}{\it Proprietary Models} \\
GPT-4o
& 34.0 & & 61.2 & & 75.4 & & - & & - & & 71.9 & \\
Gemini-2.0-Flash
& 45.4 & & - & & - & & - & & - & & - & \\
GPT-5
& 55.0 & & 84.6 & & 82.6 & & 74.1 & & - & & 83.3 & \\
GPT-5.4
& - & & 87.1 & & - & & 78.3 & & - & & - & \\
Gemini-3-Pro-Preview
& 52.5 & & 87.6 & & - & & - & & - & & - & \\

\midrule
\rowcolor{l4color!30}
\multicolumn{13}{c}{\it Qwen2.5-VL-based Models} \\
Video-R1-7B
& 34.6 & & 49.8 & & 64.2 & & 62.7 & & 72.6 & & 57.4 & \\
Video-RTS-7B
& 33.5 & & 52.7 & & 66.4 & & 62.7 & & 74.9 & & 63.0 & \\
VideoRFT-7B
& 36.8 & & 51.1 & & \textbf{68.5} & & 62.1 & & 73.7 & & 59.8 & \\
ViSS-R1-7B
& 37.3 & & 51.7 & & 66.1 & & 65.6 & & 75.3 & & 60.5 & \\
Video-RAISE-7B
& 34.8 & & 52.8 & & 66.4 & & 64.3 & & 74.7 & & 57.4 & \\
Video-KTR-7B
& 37.5 & & 51.3 & & 65.7 & & 64.8 & & 73.3 & & 57.3 & \\

\addlinespace[3pt]
Qwen2.5-VL-7B-Instruct
& 32.5 & & 46.1 & & 59.2 & & 59.1 & & 70.4 & & 50.6 & \\
+SAGE
& \textbf{37.9} & {\color[HTML]{C00000}\scriptsize (+5.4)}
& \textbf{53.4} & {\color[HTML]{C00000}\scriptsize (+7.3)}
& 65.8 & {\color[HTML]{C00000}\scriptsize (+6.6)}
& \textbf{68.1} & {\color[HTML]{C00000}\scriptsize (+9.0)}
& \textbf{78.4} & {\color[HTML]{C00000}\scriptsize (+8.0)}
& \textbf{66.5} & {\color[HTML]{C00000}\scriptsize (+15.9)} \\

\midrule
\rowcolor{l5color!30}
\multicolumn{13}{c}{\it Qwen3-VL-based Models} \\
VideoSSR-8B
& 40.3 & & 63.1 & & 71.2 & & 68.9 & & \textbf{82.9} & & 67.6 & \\
VideoAuto-R1-8B
& 42.0 & & 65.0 & & 71.0 & & 71.8 & & 80.6 & & 67.9 & \\
OneThinker-8B
& 41.0 & & 66.2 & & 70.5 & & 69.4 & & 80.1 & & 66.5 & \\
\addlinespace[3pt]
Qwen3-VL-4B-Instruct
& 44.2 & & 53.0 & & 62.7 & & 58.3 & & 71.6 & & 51.6 & \\
+SAGE
& 45.8 & {\color[HTML]{C00000}\scriptsize (+1.6)}
& 63.2 & {\color[HTML]{C00000}\scriptsize (+10.2)}
& 71.2 & {\color[HTML]{C00000}\scriptsize (+8.5)}
& 68.5 & {\color[HTML]{C00000}\scriptsize (+10.2)}
& 78.7 & {\color[HTML]{C00000}\scriptsize (+7.1)}
& 67.9 & {\color[HTML]{C00000}\scriptsize (+16.3)} \\
\addlinespace[3pt]
Qwen3-VL-8B-Instruct
& 41.3 & & 57.8 & & 66.7 & & 60.3 & & 75.6 & & 54.3 & \\
+SAGE
& \textbf{48.0} & {\color[HTML]{C00000}\scriptsize (+6.7)}
& \textbf{69.6} & {\color[HTML]{C00000}\scriptsize (+11.8)}
& \textbf{75.4} & {\color[HTML]{C00000}\scriptsize (+8.7)}
& \textbf{74.8} & {\color[HTML]{C00000}\scriptsize (+14.5)}
& 82.6 & {\color[HTML]{C00000}\scriptsize (+7.0)}
& \textbf{70.0} & {\color[HTML]{C00000}\scriptsize (+15.7)} \\
\bottomrule
\end{tabular*}
\end{table}

\paragraph{Spatial reasoning.}
Table~\ref{tab:spatial_comparison} shows consistent improvements across all seven benchmarks and three backbones, with larger gains on more demanding spatial reasoning tasks. CV-Bench shows relatively modest improvements, as the baselines are already strong on basic spatial judgments, while SPAR, ViewSpatial, MMSI, and OmniSpatial exhibit more substantial gains on structured, viewpoint-dependent, and multi-image reasoning. For Qwen3, the gains on MindCube and OST exceed 20 points, highlighting the benefit of the proposed training for integrating spatial evidence across multiple views or observations. Overall, duality training improves not only basic spatial recognition but also more complex relational reasoning. Notably, the SAGE-trained 4B model consistently outperforms the untuned 8B model across all seven benchmarks, demonstrating the benefit of targeted relational supervision beyond model scaling.

\begin{table}[!t]
\centering
\caption{Performance comparison on spatial reasoning benchmarks.}
\label{tab:spatial_comparison}
\small
\setlength{\tabcolsep}{2pt}
\begin{tabular*}{\textwidth}{@{}l@{\hspace{0.6em}\extracolsep{\fill}}*{7}{c@{\,}l}@{}}
\toprule
Model & \makebox[0pt][c]{CV-Bench} & & \makebox[0pt][c]{SPAR} & & \makebox[0pt][c]{ViewSp.} & & \makebox[0pt][c]{MMSI} & & \makebox[0pt][c]{MindCube} & & \makebox[0pt][c]{OST} & & \makebox[0pt][c]{OmniSp.} & \\
\midrule
\rowcolor{l1color!30}
\multicolumn{15}{c}{\it Proprietary Models} \\
GPT-4o
& 75.4 & & 36.4 & & 32.6 & & 30.3 & & 38.8 & & - & & - & \\
\midrule
\rowcolor{l2color!30}
\multicolumn{15}{c}{\it Other Open-source Backbones} \\
InternVL-2.5-8B
& 76.5 & & 36.3 & & 43.2 & & 25.7 & & 18.7 & & 20.4 & & 24.3 & \\
LLaVA-OV-7B
& 58.3 & & 31.2 & & 27.5 & & 24.5 & & 47.3 & & 19.5 & & 23.1 & \\
SenseNova-SI-8B
& 82.4 & & 42.1 & & \textbf{54.7} & & \textbf{43.3} & & 47.6 & & 43.8 & & 27.5 & \\
\midrule
\rowcolor{l4color!30}
\multicolumn{15}{c}{\it Qwen2.5-VL-based Models} \\
SpatialLadder-3B
& 73.7 & & 34.4 & & 44.2 & & 29.2 & & 43.4 & & 23.9 & & 30.8 & \\
VST-7B-RL
& 86.5 & & 32.8 & & 43.5 & & 34.8 & & 44.9 & & 38.0 & & 32.2 & \\
\addlinespace[3pt]
Qwen2.5-VL-7B-Instruct
& 74.2 & & 21.8 & & 37.4 & & 26.4 & & 29.5 & & 38.1 & & 23.7 & \\
+SAGE
& 77.2 & {\color[HTML]{C00000}\scriptsize (+3.0)}
& 34.4 & {\color[HTML]{C00000}\scriptsize (+12.6)}
& 44.1 & {\color[HTML]{C00000}\scriptsize (+6.7)}
& 32.3 & {\color[HTML]{C00000}\scriptsize (+5.9)}
& 44.5 & {\color[HTML]{C00000}\scriptsize (+15.0)}
& 46.3 & {\color[HTML]{C00000}\scriptsize (+8.2)}
& 38.3 & {\color[HTML]{C00000}\scriptsize (+14.6)} \\
\midrule
\rowcolor{l5color!30}
\multicolumn{15}{c}{\it Qwen3-VL-based Models} \\
Qwen3-VL-4B-Instruct
& 81.8 & & 35.4 & & 41.3 & & 24.8 & & 18.0 & & 25.2 & & 30.1 & \\
+SAGE
& 83.5 & {\color[HTML]{C00000}\scriptsize (+1.7)}
& 41.4 & {\color[HTML]{C00000}\scriptsize (+6.0)}
& 47.3 & {\color[HTML]{C00000}\scriptsize (+6.0)}
& 31.9 & {\color[HTML]{C00000}\scriptsize (+7.1)}
& 39.7 & {\color[HTML]{C00000}\scriptsize (+21.7)}
& \textbf{46.4} & {\color[HTML]{C00000}\scriptsize (+21.2)}
& 39.2 & {\color[HTML]{C00000}\scriptsize (+9.1)} \\
\addlinespace[3pt]
Qwen3-VL-8B-Instruct
& 82.9 & & 37.2 & & 42.0 & & 26.2 & & 26.4 & & 25.7 & & 32.9 & \\
+SAGE
& \textbf{86.7} & {\color[HTML]{C00000}\scriptsize (+3.8)}
& \textbf{43.2} & {\color[HTML]{C00000}\scriptsize (+6.0)}
& 46.2 & {\color[HTML]{C00000}\scriptsize (+4.2)}
& 34.8 & {\color[HTML]{C00000}\scriptsize (+8.6)}
& \textbf{48.4} & {\color[HTML]{C00000}\scriptsize (+22.0)}
& 46.1 & {\color[HTML]{C00000}\scriptsize (+20.4)}
& \textbf{40.4} & {\color[HTML]{C00000}\scriptsize (+7.5)} \\
\bottomrule
\end{tabular*}
\end{table}

\subsection{Ablation Study}
\label{ablation}

\begin{table}[t]
\centering
\caption{Ablation study of duality supervision, operation design, self-evolution, and shared-model optimization.}
\label{tab:duality_supervision}
\label{tab:ablation_refined}
\label{tab:pool_ablation}
\label{tab:shared_ablation}
\small
\setlength{\tabcolsep}{2pt}
\begin{tabular*}{\textwidth}{@{}l@{\hspace{0.6em}\extracolsep{\fill}}*{6}{c}@{}}
\toprule
Variant & \makecell{Average\\Score} & Consistency & \makecell{Joint\\Correctness} & \makecell{Wrong-\\consistent} & \makecell{GPU-\\hours} & \makecell{Peak Memory\\(GiB/GPU)} \\
\midrule
Base & 43.8 & 61.4 & 24.1 & 37.3 & - & - \\
\textbf{SAGE} & 52.9 & 75.9 & 45.7 & 30.2 & 26.8 & 76.2 \\
\midrule
\rowcolor{l2color!25}
\multicolumn{7}{c}{\textit{A. Reward design (fixed paired stream)}} \\
Original-only GRPO & 47.7 & 62.3 & 28.8 & 33.5 & 18.7 & 73.1 \\
Dual augmentation only & 51.4 & 68.9 & 35.3 & 33.6 & 37.6 & 76.2 \\
Ungated consistency & 49.4 & 81.1 & 38.4 & 42.7 & 26.4 & 76.2 \\
\midrule
\rowcolor{l4color!30}
\multicolumn{7}{c}{\textit{B. Operation families}} \\
Visual only & 48.8 & 74.8 & 39.3 & 35.5 & 23.4 & 77.0 \\
Linguistic only & 45.1 & 73.5 & 35.6 & 37.9 & 22.5 & 76.2 \\
\midrule
\rowcolor{l1color!30}
\multicolumn{7}{c}{\textit{C. Operation selection and discovery}} \\
Fixed-all & 49.3 & 76.6 & 40.3 & 36.3 & 42.8 & 76.2 \\
Random-$K$ & 47.0 & 74.1 & 37.7 & 36.4 & 26.6 & 76.1 \\
Adaptive-$K$ & 49.8 & 76.8 & 41.0 & 35.8 & 27.4 & 76.2 \\
\midrule
\rowcolor{l5color!30}
\multicolumn{7}{c}{\textit{D. Online evolution}} \\
w/o online discovery & 44.6 & 74.0 & 35.7 & 38.3 & 21.6 & 76.2 \\
w/o retirement & 48.6 & 77.1 & 39.0 & 38.1 & 28.8 & 76.2 \\
w/o reactivation & 50.2 & 73.8 & 41.5 & 32.3 & 26.9 & 76.2 \\
\midrule
\rowcolor{l4color!30}
\multicolumn{7}{c}{\textit{E. Shared-model optimization}} \\
Separate proposer--solver & 52.9 & 76.3 & 45.9 & 30.4 & 39.0 & 79.8 \\
Shared, w/o projection & 51.4 & 74.7 & 42.9 & 31.8 & 27.1 & 76.2 \\
\bottomrule
\end{tabular*}
\end{table}

The results in Table~\ref{tab:duality_supervision} show that ordinary GRPO yields only limited gains, while its low semantic consistency suggests that the model does not reliably capture the underlying relations, limiting the resulting task improvement. Dual augmentation is competitive, yet updating separately on augmented inputs incurs additional optimizer steps; Appendix~\ref{app:joint_update_analysis} identifies conditions under which a joint update can avoid both this extra update cost and the order-dependent bias introduced by sequential updates. Ungated consistency achieves the highest consistency largely through wrong-consistent predictions, highlighting that agreement alone is not reliable supervision. Combining visual and linguistic operations also outperforms either family alone, suggesting complementary supervision across changes in visual presentation and question formulation.
For operation scheduling, both Fixed-all and Random-$K$ underperform SAGE, while Adaptive-$K$ simply allocates training opportunities by operation counts without probing their utility, also falling short of SAGE; removing discovery, retirement, or reactivation further reduces performance, supporting their complementary roles in expanding coverage, reallocating training effort, and revisiting forgotten relations.
Separate proposer--solver models achieve similar accuracy but duplicate trainable parameters and optimizer states, increasing GPU-hours and memory. Removing gradient projection also degrades performance, indicating that gradient decomposition helps mitigate interference and enables effective role sharing within a single model.

\subsection{Further Analysis}
\label{sec:further_analysis}

\paragraph{Operation lifecycle.}
\label{sec:pool}
\label{operation_analysis}
Figure~\ref{fig:lifecycle} shows the dynamics of selected operations. Mastered operations yield active slots to candidates exposing unresolved inconsistencies, forming an adaptive curriculum that focuses training on remaining reasoning weaknesses.

\begin{figure}[!t]
\centering
\includegraphics[width=\linewidth]{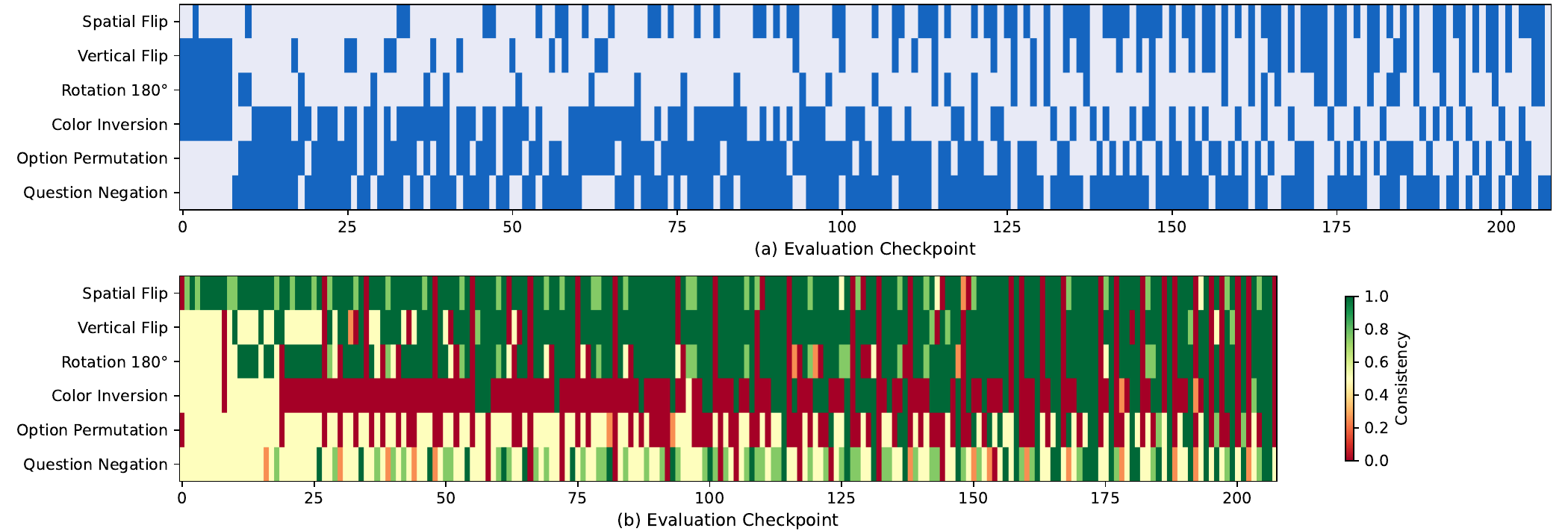}
\caption{Operation lifecycle during SAGE training. Operations are proposed, promoted to the active pool, trained until mastery, and periodically re-evaluated for forgetting.}
\label{fig:lifecycle}
\end{figure}

\paragraph{Correctness and generalization.}
\label{sec:ccww}
\label{sec:unseen}
Joint correctness rises by 21.6 points and wrong-consistent predictions fall by 7.1 points; a 20.8-point gain in unseen joint correctness over Base suggests transfer beyond the training pool. Appendix~\ref{app:further_analysis} details both evaluations.

\paragraph{Proposal quality and sensitivity.}
\label{sec:proposal}
\label{sec:sensitivity}
Generated proposals are 88.33\% valid, 70.00\% novel, and 56.67\% useful, extending supervision beyond the initial pool. Gains over Base persist across six $\lambda$ values and five $K$ settings; Appendix~\ref{app:further_analysis} reports proposal quality and sensitivity.

\FloatBarrier

\section{Conclusion}
\label{sec:conclusion}

We presented SAGE, a self-evolving framework that strengthens spatial reasoning in VLMs through geometric and linguistic duality. By discovering informative transformations and enforcing their answer relations, SAGE encourages reasoning grounded in spatial structure rather than surface patterns. Experiments across six video understanding and seven spatial reasoning benchmarks demonstrate improvements across multiple backbones, supporting transformation-aware supervision as a complementary signal for post-training. A shared proposer--solver model integrates operation discovery and reasoning within a unified learning process, with solver-priority updates mitigating interference between these roles. Future work includes extending duality to richer 3D and multi-view transformations and exploring self-evolving relational supervision beyond spatial reasoning.

\subsection*{AI use statement}
LLM-based tools assisted with writing, coding, and proof drafting. DeepSeek-V4-Pro was used to prepare proposer-format SFT examples (Section~\ref{main_setup}). The authors are responsible for verifying AI-assisted material and for the final content of this work.

\subsection*{Ethics statement}
SAGE studies spatial reasoning using existing datasets and pretrained models. Improved consistency does not guarantee factual correctness or safe behavior, and inherited biases and privacy risks may persist. Generated transformation code requires restricted execution and validation. Applications in robotics, navigation, or surveillance require additional safety and privacy assessment; benchmark gains alone do not establish deployment readiness. Use and redistribution of datasets, model weights, and images must respect their licenses and applicable restrictions.

\subsection*{Reproducibility statement}
Given resource constraints, we plan to release the training data, code, and model weights upon acceptance. The paper and appendices document the training configuration, hyperparameters, evaluation protocols, and baseline result sources to facilitate reproduction.

\bibliographystyle{iclr2027_conference}
\bibliography{iclr2027_conference}

\appendix

\newpage

\section{Related Work}
\label{sec:related}

\subsection{Vision Language Models} 

The rapid evolution of VLMs has significantly transformed multimodal perception \cite{Zhang_2023-Video-LLaMA, Maaz_2023_Video-ChatGPT, Zhang_2024_Survey, Shu_2025_Video-XL}. Foundational frameworks like CLIP \cite{Radford_2021_CLIP} and BLIP \cite{Li_2022_Blip} established cross-modal alignment, while subsequent architectures such as LLaVA \cite{Liu_2023_LLaVa}, Flamingo \cite{Alayrac_2022_Flamingo}, and Qwen3-VL \cite{Bai_2025_Qwen3-VL} advanced the field by integrating visual encoders with LLMs for sophisticated reasoning. Beyond architectural scaling, innovative strategies have emerged to refine model efficiency and perception. For instance, CoOp \cite{Zhou_2022_CoOp} introduced learnable context optimization, while VisionZip \cite{Yang_2025_VisionZip} and VisionThink \cite{Yang_2025_VisionThink} improve inference efficiency through visual token compression and reinforcement learning-based resolution scaling, respectively. Furthermore, recent works like CoFFT \cite{Zhang_2025_COFFT} and FOCUS \cite{Yang_2025_FOCUS} have pushed the boundaries of visual reasoning and interactive editing by incorporating iterative foresight-focus decoding and unified segmentation-aware perception. Despite these advancements, existing VLMs often prioritize surface-level statistical patterns over genuine logical invariances \cite{Berglund_2024_Curse}. Consequently, they frequently exhibit reasoning inconsistencies and struggle to fully comprehend complex spatial and temporal geometries, particularly when faced with Out-Of-Distribution (OOD) visual arrangements \cite{Liu_2023_Visual}.

\subsection{Reinforcement Learning for Reasoning}

The integration of Reinforcement Learning from Human Feedback (RLHF) \cite{Christiano_2017_RLHF, Stiennon_2020_Learning, Ouyang_2022_Training} has become a standard paradigm for aligning large models, with algorithms like PPO \cite{Schulman_2017_PPO} and DPO \cite{Rafailov_2023_DPO} paving the way for goal-directed optimization. Building on these, Group Relative Policy Optimization (GRPO) \cite{Shao_2024_GRPO} has emerged as a more efficient alternative by computing advantages within sampled groups, thereby eliminating the need for a separate value network and simplifying the integration of custom reward signals. Early explorations utilized VLMs as zero-shot reward models to guide agent behavior \cite{Rocamonde_2024_VLM-RMs}, while Zhai et al. \cite{Zhai_2024_Fine-Tuning} leveraged RL to fine-tune VLMs as decision-making agents via Chain-of-Thought (CoT) \cite{Wei_2022_CoT} exploration. More recently, VL-Rethinker \cite{Wang_2025_VL-Rethinker} adapted GRPO to incentivize self-reflection in multimodal reasoning, and VAPO \cite{Tian_2026_VAPO} introduced vision-anchored rewards to mitigate the "visual forgetting" typically associated with prolonged reasoning. Our work similarly adopts the GRPO framework, as its flexible reward structure allows us to incorporate geometric and linguistic duality as self-correcting signals. By verifying consistency automatically across semantically equivalent transformations, SAGE enables a self-evolving training process that enhances spatial reasoning without the burden of extensive manual annotations.

\subsection{Spatial reasoning in VLMs}

Spatial reasoning remains a persistent challenge for vision-language models, as they often struggle with 3D geometries and relative object configurations \cite{Liu_2023_Visual, Chen_2024_SpatialVLM, Chen_2025_Why}. Early efforts like SpatialRGPT \cite{Cheng_2024_SpatialRGPT} and SpatialVLM \cite{Chen_2024_SpatialVLM} addressed these limitations by introducing depth-aware plugins or internet-scale metric spatial data, while AdaptVis \cite{Chen_2025_Why} explored inference-time attention sharpening to rectify misdirected focus. Recently, the R1-style reinforcement learning paradigm has been extended to spatial domains. Video-R1 \cite{Feng_2025_Video-R1} utilized T-GRPO to incentivize temporal modeling, and ViLaSR \cite{Wu_2025_VILASR} integrated interwoven thinking with visual drawing. Furthermore, frameworks like MetaSpatial \cite{Pan_2026_MetaSpatial} and SpatialLadder \cite{Li_2026_SpatialLadder} proposed physics-aware modulation and progressive training hierarchies to bridge the gap between perception and reasoning.
More recently, SVQA-R1 \cite{Wang_2026_SVQA-R1} employed view-consistent rewards via static spatial perturbations, but its training signal remains static and cannot adapt to the model’s evolving weaknesses. While Faithful GRPO \cite{Kancheti_2026_Faithful} and SpatialEvo \cite{Li_2026_SpatialEvo} introduce forms of self-correction or evolving environments, they are not designed to enforce transformation-level duality consistency. 
As a result, these methods fail to guarantee that the underlying spatial logic is truly mastered; a model might pass these specialized constraints yet still fail on fundamental dual transformations, leading to poor generalization on OOD tasks.

\subsection{Self-evolving reasoning systems}

Self-evolving reasoning has emerged as a promising approach to reducing dependence on fixed training curricula by adapting supervision to the learner's capabilities. R-Zero \cite{Huang_2025_RZero} co-evolves separate Challenger and Solver models, while J-Zero \cite{Chu_2026_JZero} introduces a co-adapting Judge to refine the evaluation signal. Complementary approaches ground self-improvement in executable environments or visual observations: Absolute Zero \cite{Zhao_2025_AbsoluteZero} jointly trains a shared proposer--solver on code reasoning tasks, whereas VisPlay \cite{He_2026_VisPlay} develops visual reasoning through question generation over unlabeled images. Although zero-data self-evolution offers an elegant route to autonomous learning, task diversity can remain constrained by the generator's learned priors. R-Zero reports degradation over successive iterations and identifies diversity loss and bias amplification as possible contributors. Extending zero-data self-evolution to multimodal tasks may additionally require image or audio generators to create the observations themselves. The resulting dependence on generation quality and cross-modal alignment can limit training effectiveness, particularly when spatial reasoning requires precise geometric relations and reliable answer labels. SAGE instead grounds operation discovery in existing labeled images and videos, proposing geometric and linguistic transformations with explicit answer mappings. This anchors scene content and supervision to source examples while allowing the training curriculum to evolve with the model's weaknesses, without requiring a separate image generator or learned judge. By probing the same underlying relations across different presentations, SAGE is designed to discourage presentation-specific shortcuts and promote spatial generalization, rather than relying solely on diversity generated from the model's own knowledge.

\section{Additional Method Details}
\label{sec:method_details}

\subsection{Operation Families and Composition}
\label{app:operation_validation}

\paragraph{Executable validation.}
The parameter-free verifier checks parsing, restricted compilation, output type and shape, determinism, duplicates, the declared answer map, and applicability exclusions. For example, it rejects an identity answer map for a left/right flip. These checks provide operational acceptance, not a proof of semantic validity; a valid duality operation must still satisfy Eq.~\eqref{eq:duality_axiom} and have bijective $T_{\mathcal A}$. Surviving candidates are evaluated using the pair counts and scores in Eq.~\eqref{eq:pair_scores}.

\paragraph{Visual duality.}
The auditable seed operations include: (i) horizontal flip, for which $T_{\mathcal A}$ swaps left$\leftrightarrow$right; (ii) vertical flip, which swaps top$\leftrightarrow$bottom; (iii) $180^\circ$ rotation, whose answer map composes the horizontal and vertical mappings; and (iv) appearance perturbations, for which $T_{\mathcal A}=\mathrm{id}$ only when the applicability predicate excludes appearance-dependent questions. Flips also exclude text reading and other orientation-sensitive samples unless the proposal supplies a valid answer map.

\paragraph{Linguistic duality.}
For option permutation, transformed position $k$ contains original option $\pi(k)$ for $\pi\in S_C$, so $T_{\mathcal A}(a)=\pi^{-1}(a)$. Binary logical negation maps the answer to its complement and excludes ambiguous or non-binary questions. Semantic paraphrase uses $T_{\mathcal A}=\mathrm{id}$ and compares outputs with the task's deterministic answer extractor rather than raw strings.

\paragraph{Composition.}
If $\mathcal T_1=(T_{\mathcal X,1},T_{\mathcal A,1},\mathcal S_1)$ and $\mathcal T_2=(T_{\mathcal X,2},T_{\mathcal A,2},\mathcal S_2)$ satisfy Eq.~\eqref{eq:duality_axiom}, their composition uses $T_{\mathcal X,1}\circ T_{\mathcal X,2}$, $T_{\mathcal A,1}\circ T_{\mathcal A,2}$, and domain $\mathcal S_2\cap T_{\mathcal X,2}^{-1}(\mathcal S_1)$. This preserves the declared answer relation on the restricted domain, but executable and probe checks remain necessary because valid code does not prove semantic validity.

\subsection{Training Iteration and Pool Updates}

For a horizontal-reflection proposal, the verifier checks the input transformation, the left/right answer mapping, and the applicability predicate. The model is then evaluated on applicable original--dual pairs. Their four-way correctness profile determines the one-sided failure utility $U$ and joint correctness $J$; a candidate is admitted only with sufficient probe support and utility.

For a question asking where a cup lies relative to a plate, a source answer of ``left'' maps to ``right'' after reflection. Both inputs receive task supervision. If only one answer is correct, the pair receives no consistency bonus, but neither side is discarded from the task objective. Each prompt's advantages are normalized within its own rollout group. After the shared policy update, supported probes refresh the operation profiles using Eq.~\eqref{eq:profile_ema}. Retirement requires both $U\leq\tau_{\rm lo}$ and $J\geq\kappa_{\rm hi}$; a mastered operation returns to the candidate pool when $U>\tau_{\rm hi}$ or $J<\kappa_{\rm lo}$.

\section{Pseudocode}
\label{pseudocode}

In this section, we present the pseudocode of SAGE, as shown in Algorithm~\ref{alg:self_evolving}.

\begin{algorithm}[t]
\caption{Shared-Parameter Generative SAGE}
\label{alg:self_evolving}
\begin{algorithmic}[1]
\STATE \textbf{Input:} proposer-format-SFT policy $\mathcal M_{\theta_0}$; $\mathcal D,\mathcal D_{\rm probe},\mathcal P_{\rm seed}$; $T_{\rm train}$; even $G\geq4$; $E,E_p,M,K$; $\lambda,\beta,\alpha,\eta,\rho,p_f,\epsilon$; $N_{\min},\tau_{\rm add},\tau_{\rm lo},\tau_{\rm hi},\kappa_{\rm lo},\kappa_{\rm hi}$
\STATE $\mathcal P_0\leftarrow\textsc{ValidateSeeds}(\mathcal P_{\rm seed})$; probe seeds and retain those with $N_i\geq N_{\min}$
\STATE Initialize profiles and states; set $b_0\leftarrow0$ and derive $\mathcal O_0$ via Eq.~\eqref{eq:promote}
\FOR{step $t = 0, 1, \ldots, T_{\rm train}-1$}
    \STATE $\widehat g_{{\rm prop},t}\leftarrow0$; $b_{t+1}\leftarrow b_t$
    \IF{$(t+1) \bmod E_p=0$}
        \STATE $h_t\leftarrow\textsc{ProposerContext}(\mathcal P_t,\textsc{Profile}_t)$; sample $y_t\sim\mathcal M_{\theta_t}(\cdot\mid h_t)$
        \STATE $\widetilde{\mathcal T}_t\leftarrow\textsc{Parse}(y_t)$; $\nu_t\leftarrow\textsc{ExecutableChecks}(\widetilde{\mathcal T}_t)$
        \STATE $(N_t,U_t,J_t)\leftarrow(0,0,0)$
        \IF{$\nu_t=1$}
            \STATE $(\mathrm{CC},\mathrm{CW},\mathrm{WC},\mathrm{WW})\leftarrow\textsc{FourWayProbe}(\mathcal M_{\theta_t},\widetilde{\mathcal T}_t,\mathcal D_{\rm probe})$
            \STATE Compute $N_t$; if $N_t>0$, compute $U_t,J_t$
        \ENDIF
        \STATE Compute $u_t$, $\widehat g_{{\rm prop},t}$, and $b_{t+1}$ via Eqs.~\eqref{eq:proposal_utility}--\eqref{eq:proposal_objective}
        \IF{$\nu_t=1$ and $N_t\geq N_{\min}$ and $U_t\geq\tau_{\rm add}$}
            \STATE $\mathcal P_t\leftarrow\textsc{BoundedInsert}(\mathcal P_t,\widetilde{\mathcal T}_t,\textsc{Candidate},M)$ with its admission profile
        \ENDIF
    \ENDIF
    \STATE Sample $(x,a^*)\sim\mathcal D$ and generate $\{o_i\}_{i=1}^G\sim\mathcal M_{\theta_t}(\cdot\mid x)$
    \STATE $\mathcal O_t^x\leftarrow\{\mathcal T_i=(T_{\mathcal X,i},T_{\mathcal A,i},\mathcal S_i)\in\mathcal O_t:x\in\mathcal S_i\}$
    \IF{$\mathcal O_t^x\neq\emptyset$}
        \STATE Sample $\mathcal T_i\in\mathcal O_t^x$ proportional to $w_i$, uniformly if all priorities are zero
        \STATE $\{o'_j\}_{j=1}^{G/2}\sim\mathcal M_{\theta_t}(\cdot\mid T_{\mathcal X,i}(x))$
        \STATE Compute $r(o_i)$ and $r'(o'_j)$ via Eqs.~\eqref{eq:total_reward}--\eqref{eq:consistency_reward}
        \STATE Compute the two-sided $\mathcal J_{\rm solve}$ and $g_{\rm solve}=\nabla_\theta\mathcal J_{\rm solve}$ via Eq.~\eqref{eq:grpo_loss}
    \ELSE
        \STATE Set $r_{\rm cons}(o_i)=0$; compute original-only $\mathcal J_{\rm solve}$ and $g_{\rm solve}$
    \ENDIF
    \STATE Set $g_{\rm prop}=\widehat g_{{\rm prop},t}$; compute $\widetilde g_{\rm prop}$ and update shared $\theta_{t+1}$ via Eqs.~\eqref{eq:projected_gradient}--\eqref{eq:joint_objective}
    \IF{$(t+1) \bmod E=0$}
        \STATE Using $\mathcal M_{\theta_{t+1}}$, probe every operation in $\mathcal O_t$ and update its profile via Eq.~\eqref{eq:profile_ema}
        \STATE If any \textsc{Mastered} operation exists, probe one with probability $p_f$ and update its profile via Eq.~\eqref{eq:profile_ema}
        \STATE Apply Eqs.~\eqref{eq:retire}--\eqref{eq:forget}; fill vacant active slots with highest-priority candidates via Eq.~\eqref{eq:promote}
    \ENDIF
    \STATE Set $\mathcal P_{t+1}\leftarrow\mathcal P_t$ and $\mathcal O_{t+1}\leftarrow\{\mathcal T_i\in\mathcal P_{t+1}:s_i=\textsc{Active}\}$
\ENDFOR
\end{algorithmic}
\end{algorithm}

\section{Experimental Setup Details}
\label{setup}

\subsection{Benchmark Descriptions}

We evaluate SAGE on 13 benchmarks spanning video understanding and spatial reasoning. Below we describe each benchmark in detail, along with representative examples.

\subsubsection{Video Understanding Benchmarks}

\begin{itemize}
    \item \textbf{MVBench}~\cite{Li_2024_MVBench} (4,000 samples). A comprehensive multi-modal video understanding benchmark covering 20 temporal understanding tasks, including action sequence recognition, scene transition detection, object tracking, and counterfactual inference. All questions are multiple-choice. \emph{Example:} ``What is the action performed by the person in the video?'' with options describing different activities.

    \item \textbf{TempCompass}~\cite{Liu_2024_TempCompass} (7,540 samples). A temporal video description matching benchmark designed to test whether VLMs genuinely understand temporal dynamics or merely rely on static frame features. Questions probe fine-grained temporal perception such as speed, order, and direction of events.

    \item \textbf{VideoMME}~\cite{Fu_2025_Video-MME} (2,700 samples). A large-scale evaluation benchmark covering diverse video understanding scenarios across short ($<$2min), medium (4--15min), and long ($>$30min) videos, with questions spanning perception, cognition, and knowledge domains.

    \item \textbf{VideoMMMU}~\cite{Hu_2025_Video-MMMU} (900 samples). A challenging benchmark requiring domain-specific expert knowledge from video content, covering disciplines such as science, engineering, medicine, and humanities.

    \item \textbf{VSI-Bench}~\cite{Yang_2025_VSI-Bench} (5,130 samples). A visual spatial intelligence benchmark constructed from 288 indoor video scenes. It contains both multiple-choice questions (2,490) and numerical regression tasks (2,640) across 10 sub-categories: object counting, absolute distance estimation, relative distance comparison, object size estimation, room size estimation, relative direction (easy/medium/hard), appearance order, and route planning. Following the original benchmark protocol, we report the average score across all 10 sub-tasks, using accuracy for multiple-choice and mean relative accuracy (MRA) for regression tasks. \emph{Example:} ``If I am standing by the refrigerator and facing the table, is the dishwasher to my front-left, front-right, back-left, or back-right?''

    \item \textbf{MMVU}~\cite{Zhao_2025_MMVU} (625 samples). A multi-discipline video understanding benchmark with questions from domains including science, healthcare, engineering, and daily life, requiring both visual understanding and domain knowledge.
\end{itemize}

\begin{figure}[t]
\centering
\includegraphics[width=\linewidth]{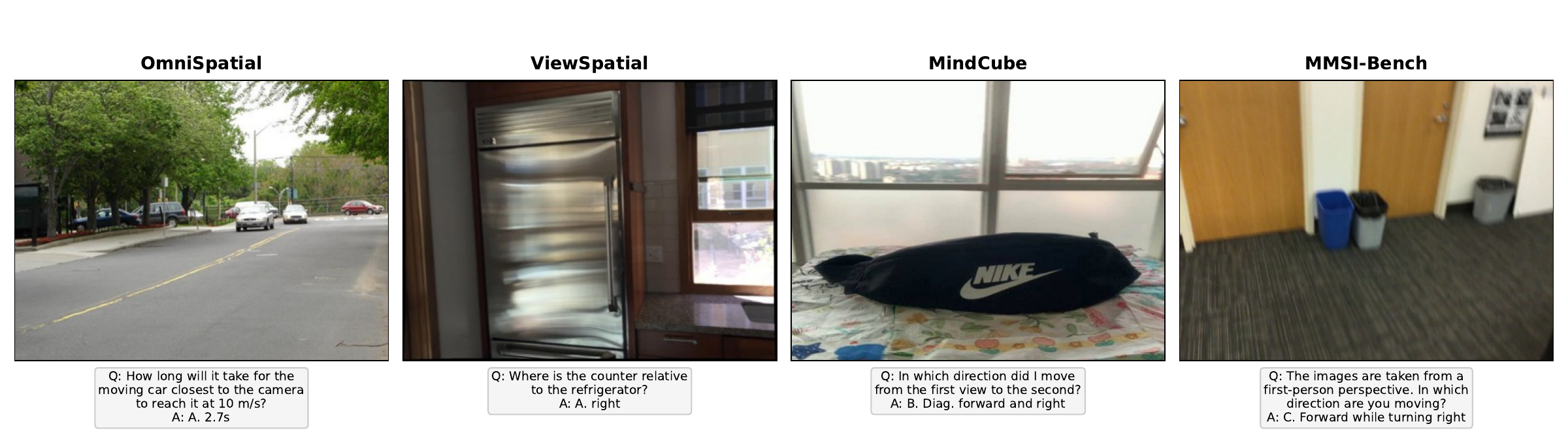}
\caption{Representative examples from four spatial reasoning benchmarks. Each example shows the input image, the spatial question, and the ground-truth answer. These questions require understanding of object distances (OmniSpatial), relative positions (ViewSpatial), egomotion and perspective changes (MindCube), and first-person movement direction (MMSI-Bench).}
\label{fig:benchmark_examples}
\end{figure}

\subsubsection{Spatial Reasoning Benchmarks}

\begin{itemize}
    \item \textbf{CV-Bench}~\cite{Tong_2024_CV-Bench} (2,638 samples). Transforms classic computer vision tasks into VQA format, with a 2D subset (1,438 samples) evaluating spatial relationships and object counting, and a 3D subset (1,200 samples) evaluating depth ordering and relative distance. Images are sourced from ADE20K, COCO, and OMNI3D with manually verified annotations.

    \item \textbf{SPAR-Bench}~\cite{Zhang_2026_SPAR-Bench} (7,207 samples). A large-scale 3D spatial perception and reasoning benchmark spanning 33 spatial tasks across single-view, multi-view, and video settings, sourced from 4,500 real-world scenes. Tasks include distance estimation, object-to-object spatial relations, spatial imagination, and route planning.

    \item \textbf{ViewSpatial-Bench}~\cite{Li_2025_ViewSpatial-Bench} (5,712 samples). Evaluates multi-perspective spatial localization by requiring models to reason about the relative positions of objects from different camera viewpoints. Questions ask about directional relationships (e.g., left, right, front, back) between objects in 3D scenes. \emph{Example:} ``Could you tell me the location of the counter in comparison to the refrigerator?'' with options: A.~right, B.~front-up, C.~back-left, D.~front.

    \item \textbf{MMSI-Bench}~\cite{Yang_2026_MMSI-Bench} (1,000 samples). A human-annotated benchmark for multi-image spatial intelligence, where each question involves two or more images from real-world scenes including autonomous driving, robotic manipulation, and indoor scanning. Covers six types of positional relationships, two types of attribute reasoning, two types of motion reasoning, and multi-step reasoning. \emph{Example:} ``The images are taken continuously from a first-person perspective. In which direction are you moving?''

    \item \textbf{MindCube}~\cite{Yin_2025_MindCube} (21,154 samples). Evaluates spatial mental modeling from limited visual views with questions testing cognitive mapping, perspective-taking, and mental simulation across 3,268 images. The model must form a coherent spatial mental model from partial observations to answer questions about unseen perspectives. \emph{Example:} ``Based on these two views showing the same scene: in which direction did I move from the first view to the second view?''

    \item \textbf{OST-Bench}~\cite{Lin_2026_OST-Bench} (10,165 QA turns across 1.4K scenes; we evaluate 3,194 questions reformulated as independent single-turn instances). Tests online spatio-temporal scene understanding by requiring models to track object existence, positions, and state changes across sequential observations in 3D environments. Questions probe whether models can maintain a consistent internal representation of the evolving scene.

    \item \textbf{OmniSpatial}~\cite{Jia_2026_OmniSpatial} (8,400+ QA pairs in total; we evaluate on the 1,533-sample test split). A comprehensive spatial reasoning benchmark with 50 fine-grained subcategories organized into four major categories: dynamic reasoning (e.g., motion analysis, collision prediction), complex spatial logic (e.g., multi-step reasoning), spatial interaction (e.g., navigation, traffic analysis), and perspective-taking (e.g., viewpoint changes). The test split is disjoint from the training data used in our experiments. \emph{Example:} ``How long will it take for the moving car closest to the camera to reach it if it is going at 10~m/s?''
\end{itemize}

\FloatBarrier
\subsubsection{Benchmark Metrics}
\label{app:benchmark_metrics}

Table~\ref{tab:benchmark_metrics} specifies the evaluation scope and metric for every main-table benchmark. Scores are reported as percentages, and higher is better. Unless stated otherwise, the benchmark parser extracts a canonical answer and reports exact answer or option accuracy; an extraction failure receives zero credit.

\begin{table}[htbp]
\centering
\caption{Evaluation scope and metrics for the 13 benchmarks.}
\label{tab:benchmark_metrics}
\scriptsize
\setlength{\tabcolsep}{3pt}
\begin{tabular}{@{}p{0.19\linewidth}p{0.25\linewidth}p{0.50\linewidth}@{}}
\toprule
Benchmark & Evaluated scope & Reported metric \\
\midrule
VSI-Bench & 2,490 MC + 2,640 regression & Equal average of MC accuracy and regression mean relative accuracy (MRA). \\
VideoMMMU & 882 MC + 18 numerical & Mean per-question score: exact option accuracy for MC and normalized numerical exact match for numerical questions. \\
MMVU (MC) & 625 validation MC questions & Exact option accuracy; the 375 validation open-ended questions are excluded. \\
MVBench & 4,000 MC questions, 20 tasks & Exact option accuracy over all questions. \\
TempCompass & 7,540 converted MC questions & Exact option accuracy over all questions. \\
VideoMME & 2,700 MC questions & Exact option accuracy over all questions. \\
\midrule
CV-Bench & 1,438 2D + 1,200 3D questions & Exact option accuracy over the combined 2D and 3D subsets. \\
SPAR-Bench & 7,211 released test records & Mean per-question task score: MRA for numerical depth/distance tasks and canonical answer accuracy otherwise. \\
ViewSpatial-Bench & 5,712 questions & Exact canonical answer accuracy. \\
MMSI-Bench & 1,000 MC questions & Exact option accuracy. \\
MindCube & 21,154 questions & Exact canonical answer accuracy. \\
OST-Bench & 3,194 reformulated single-turn questions & Mean task score using option judgement, relative-error estimation, or exact enumeration according to question type. \\
OmniSpatial & 1,533 test questions & Exact canonical answer accuracy. \\
\bottomrule
\end{tabular}
\end{table}
\FloatBarrier

For VideoMMMU numerical questions, predictions are normalized as numbers and compared after rounding to two decimal places. VSI-Bench regression, SPAR numerical questions, and OST estimation use ten relative-error thresholds from $0.50$ to $0.95$, yielding fractional per-item credit in increments of $0.1$; their precise boundary and zero-target conventions are given in Appendix~\ref{app:paired_metric_protocol}. The final SPAR and OST scores average their task-specific per-question scores rather than binarizing fractional credit. \emph{Average Score} in the ablation table is the unweighted mean of the 13 full-benchmark scores.

The SPAR-Bench paper and repository describe the finalized benchmark as containing 7,207 QA pairs, whereas the currently distributed Hugging Face test split contains 7,211 records. Our evaluation follows the distributed split; this four-record difference reflects the released dataset version rather than additional filtering by our pipeline.

OST-Bench is originally organized as 10,165 QA turns in multi-turn scene dialogues. We reformulate 3,194 questions as independent single-turn instances for evaluation. All in-house OST results in the main table use this scope and therefore do not represent performance under the original multi-turn protocol.

These benchmark scores are distinct from the paired duality diagnostics. Accuracy/MRA evaluates each answer against its task target, whereas consistency evaluates whether an original--dual pair obeys its declared answer map. Appendix~\ref{app:paired_metric_protocol} defines consistency, joint correctness, wrong-consistency, and the treatment of score-thresholded numerical answers.

\subsection{Baseline Result Sources and Evaluation Settings}
\label{sec:baseline_sources}

Our comparisons combine literature-reported scores and our own evaluations because reproducing every baseline is computationally expensive. We distinguish the source of each model--benchmark result: a literature-reported score is taken from a cited paper, whereas a reproduced score is obtained by running the corresponding model in our evaluation pipeline. Reproduced evaluation does not necessarily imply retraining the baseline. Scores from different sources may use different inference settings and are not treated as controlled comparisons.

\paragraph{Result provenance.}
Table~\ref{tab:baseline_sources} summarizes result sources. Attribution is made per model--benchmark entry: reproducing one benchmark does not change the provenance of the remaining reported scores in that row. The following paragraphs give source-specific settings and coverage.

\begin{table}[t]
\centering
\caption{Result provenance for the main comparison tables.}
\label{tab:baseline_sources}
\small
\begin{tabular}{p{0.38\linewidth}p{0.17\linewidth}p{0.36\linewidth}}
\toprule
Models & Source & Scope / reference \\
\midrule
Qwen2.5-VL-7B-Instruct and +SAGE & Our evaluations & Main video and spatial results; settings in Appendix~\ref{app:training_config}. \\
Qwen3-VL-4B/8B-Instruct and +SAGE & Our evaluations & Both backbone and +SAGE scores cover all six video and seven spatial benchmarks. \\
Video-R1, Video-RTS, VideoRFT, ViSS-R1 & Reported & Video results attributed to the respective papers~\cite{Feng_2025_Video-R1,Wang_2025_Video-RTS,Wang_2025_VideoRFT,Fang_2025_ViSS-R1}. \\
Proprietary VLMs and additional video baselines & Reported & Source-specific settings and coverage are documented below and in Table~\ref{tab:additional_video_baselines}. \\
InternVL-2.5-8B, LLaVA-OV-7B, and SpatialLadder-3B & Reported and our evaluations & The first five spatial benchmarks use values reported by SpatialLadder~\cite{Li_2026_SpatialLadder}; OST and OmniSpatial are our evaluations. \\
Video-RAISE-7B & Reported & Collected from Table~1 of Xi et al.~\cite{Xi_2026_VideoRAISE}, using the 16-frame row. \\
Video-KTR-7B & Reported and our evaluations & VideoMMMU, MMVU, TempCompass, and VideoMME are from Table~1 of Wang et al.~\cite{Wang_2026_VideoKTR}, using the 16-frame row; VSI-Bench and MVBench are our evaluations. \\
VideoSSR, VideoAuto-R1, and OneThinker & Reported and our evaluations & Literature-reported cells and inference settings appear in Table~\ref{tab:qwen3_reported_settings}; the remaining populated cells are our evaluations. \\
VST-7B-RL & Reported and our evaluations & CV-Bench and MMSI are from Table~2 of Yang et al.~\cite{Yang_2025_VST}; the remaining populated cells are our evaluations. \\
SenseNova-SI-8B & Reported and our evaluations & InternVL3-8B variant; ViewSpatial and MMSI are from Table~1 of Cai et al.~\cite{Cai_2026_SenseNovaSI}, and the remaining populated cells are our evaluations. MindCube-Tiny is excluded. \\
\bottomrule
\end{tabular}
\end{table}

\paragraph{Frontier video baseline provenance.}
All newly added GPT-5, GPT-5.4, and Gemini-3-Pro-Preview scores in Table~\ref{tab:simplified_results} are literature-reported, not our reproductions. For GPT-5, VSI-Bench 55.0 comes from SenseNova-SI (arXiv:2511.13719v2, Table~1; GPT-5-2025-08-07)~\cite{Cai_2026_SenseNovaSI}; VideoMMMU 84.6 and MMVU (MC) 82.6 come from Video-KTR Table~1~\cite{Wang_2026_VideoKTR}; MVBench 74.1 and VideoMME 83.3 come from CoPE-VideoLM Tables~3 and~2, respectively, with the latter using no subtitles. For GPT-5.4, VideoMMMU 87.1 and MVBench 78.3 come from VideoKR Table~3 (64 frames).

For Gemini-3-Pro-Preview, VSI-Bench 52.5 is explicitly attributed to the Preview variant in SenseNova-SI Table~1, and VideoMMMU 87.6 is also reported in VideoKR Table~3 for the November 2025 Gemini 3 Pro model. The Preview designation follows the explicitly named SenseNova-SI entry; VideoKR shortens the model name to Gemini 3 Pro and does not specify an exact API snapshot. Other scores labelled only Gemini 3 Pro are withheld pending version verification. MMVU scores with an unresolved match to our MC split, TempCompass scores with an unresolved match to our converted evaluation set, and GPT-5.4 VideoMME with unverified subtitle settings remain unfilled. Each populated cell retains its source's evaluation settings; an unspecified API snapshot or inference parameter is not inferred from another source.

\paragraph{Spatial result provenance.}
VST-7B-RL uses Qwen2.5-VL-7B without an additional specialist 3D encoder. Its CV-Bench (86.5) and MMSI (34.8) scores are reported values from Table~2 of Yang et al.~\cite{Yang_2025_VST} (arXiv:2511.05491v1), where neither is designated as a restricted subset; the other five populated cells are our evaluations. VST uses 4.1M perception samples, 135K reasoning samples, and additional general multimodal data, so it is not a training-budget-matched comparison.

For GPT-4o, InternVL-2.5-8B, LLaVA-OV-7B, and SpatialLadder-3B, the first five benchmark scores in Table~\ref{tab:spatial_comparison} retain their literature attribution to SpatialLadder~\cite{Li_2026_SpatialLadder}; populated OST and OmniSpatial cells are our evaluations. For SenseNova-SI-8B, ViewSpatial and MMSI are reported values, while the other populated cells are our evaluations. SenseNova-SI-8B denotes the InternVL3-8B variant, not a Qwen3 checkpoint. All Qwen2.5-VL-7B-Instruct and Qwen3-VL-4B/8B-Instruct rows, including their +SAGE variants, are from our evaluations; each improvement is computed only against its matched same-backbone baseline.

\paragraph{Source documentation and benchmark coverage.}
Table~\ref{tab:qwen3_reported_settings} documents nine literature-reported scores and their source-specific inference settings. A single frame configuration is retained per method; these are contextual comparisons, not compute-matched runs. The remaining populated cells for VideoSSR, VideoAuto-R1, and OneThinker, together with all Qwen3-VL-4B/8B-Instruct and +SAGE rows, are from our evaluations.

\begin{table}[tbp]
\centering
\caption{Sources and inference settings of reported Qwen3-based baseline results.}
\label{tab:qwen3_reported_settings}
\small
\setlength{\tabcolsep}{3pt}
\begin{tabular}{@{}p{0.18\linewidth}p{0.34\linewidth}p{0.43\linewidth}@{}}
\toprule
Method & Reported scores / source & Source inference settings \\
\midrule
VideoSSR-8B & MVBench 68.9; VideoMME 67.6. He et al., \href{https://arxiv.org/abs/2511.06281}{\emph{VideoSSR: Video Self-Supervised Reinforcement Learning}}, Table 2 (v1). & 64-frame row; 2 FPS; $512\times512$ pixel budget; greedy direct answering without CoT; VideoMME without subtitles. \\
\addlinespace
VideoAuto-R1-8B & VideoMMMU 65.0; MMVU 71.0; MVBench 71.8; VideoMME 67.9. Liu et al., \href{https://arxiv.org/abs/2601.05175}{\emph{VideoAuto-R1: Video Auto Reasoning via Thinking Once, Answering Twice}}, CVPR supplemental Table 7. & 64-frame row; confidence-based early exit; greedy decoding; VideoMME without subtitles. The selected row differs from the released launcher's larger default frame limits. \\
\addlinespace
OneThinker-8B & VideoMMMU 66.2; MMVU 70.5; VideoMME 66.5. Feng et al., \href{https://arxiv.org/abs/2512.03043}{\emph{OneThinker: All-in-one Reasoning Model for Image and Video}}, Table 2 (v1). & Up to 128 frames; 2 FPS; per-frame pixel cap 262,144; greedy decoding; released evaluator uses no VideoMME subtitles. \\
\bottomrule
\end{tabular}
\end{table}

OneThinker's released VideoMMMU/MMVU/VideoMME manifests contain 900/625/2,700 examples and match our problem-ID sets. VideoAuto-R1's reported configuration covers all VideoMMMU tasks, the 625-question MMVU MC validation split, all 20 MVBench tasks, and VideoMME, without a sample limit. VideoSSR reports standard MVBench and VideoMME results; its evaluator excludes missing or failed samples. Its other populated cells use our evaluation records rather than the previously inspected MC-only VideoMMMU result.

\paragraph{MMVU split and reporting scope.}
MMVU denotes all 625 multiple-choice questions in the public validation split, excluding the 375 open-ended questions. Gemini-2.0-Flash's previously collected 66.5 is an overall test score covering both question formats~\cite{Zhao_2025_MMVU}, not a matching MC result, so its entry remains ``-''. VideoSSR's 71.2 is our reproduced validation-MC result. Missing results are marked ``-'' rather than substituted from a different split. These exclusions do not imply reporting errors in the source papers.

Video-KTR (65.7; Table~1, 16 frames)~\cite{Wang_2026_VideoKTR} and OneThinker (70.5; Table~2)~\cite{Feng_2026_OneThinker} explicitly report MMVU(mc), so their values are retained as published. Their reported decimals do not match a single-run accuracy rounded to one decimal place over 625 questions. The effective denominator, aggregation, and reporting convention remain unresolved; we preserve the published values without attributing the discrepancy to an error in either paper.

Literature-reported scores retain their documented model variants, metrics, splits, and inference settings. In-house results use checkpoint-specific evaluation records. MindCube-Tiny and other restricted subsets are not substitutes for the full evaluation splits used here.

\subsection{Training Configuration}
\label{app:training_config}

\paragraph{Shared-model hyperparameters.}
Table~\ref{tab:shared_hyperparameters} summarizes the shared-model training configuration. Proposer-format SFT uses a learning rate of $10^{-5}$ for one epoch, followed by the RL configuration specified in Section~\ref{main_setup}.

\begin{table}[tbp]
\centering
\caption{Hyperparameters for shared-model SAGE training.}
\label{tab:shared_hyperparameters}
\small
\setlength{\tabcolsep}{4pt}
\begin{tabular}{@{}p{0.70\linewidth}p{0.23\linewidth}@{}}
\toprule
Hyperparameter & Value \\
\midrule
Proposal interval $E_p$ (solver steps) & 100 \\
Proposer update weight $\eta$ & 0.01 \\
KL coefficient $\beta$ & 0.04 \\
Mastered-operation spot-check probability $p_f$ & 0.20 \\
Proposal-baseline EMA momentum $\rho$ & 0.9 \\
Operation-profile EMA momentum & 0.9 \\
Maximum pool capacity $M$ & 50 \\
Probe buffer capacity (source items) & 200 \\
Maximum proposal output tokens & 1,024 \\
Maximum probe-answer output tokens & 128 \\
Advantage normalization $\epsilon$ & $10^{-4}$ \\
\midrule
Minimum applicable pairs $N_{\min}$ & 32 \\
Admission threshold $\tau_{\rm add}$ & 0.10 \\
One-sided-failure thresholds $(\tau_{\rm lo},\tau_{\rm hi})$ & $(0.05,0.10)$ \\
Joint-correctness thresholds $(\kappa_{\rm lo},\kappa_{\rm hi})$ & $(0.80,0.90)$ \\
GRPO ratio-clipping half-width & 0.20 \\
Gradient accumulation steps & 1 \\
Maximum gradient norm & 5.0 \\
Maximum prompt / completion length & 16,384 / 768 \\
Rollout temperature / top-$p$ & 1.0 / 0.95 \\
Proposer-format SFT learning rate & $10^{-5}$ \\
Proposer-format SFT epochs & 1 \\
Proposer-format SFT examples & 1,000 \\
\bottomrule
\end{tabular}
\end{table}

The probe buffer contains up to 200 source items from training rollouts; only applicable pairs count toward $N_{\min}$. Profiles are initialized from admission probes. For an operation with raw probe estimate $\widehat z_i$, each stored score $z_i\in\{U_i,J_i\}$ is updated by
\begin{equation}
z_i\leftarrow
\begin{cases}
0.9z_i+0.1\widehat z_i, & N_i\geq N_{\min},\\
z_i, & N_i<N_{\min}.
\end{cases}
\label{eq:profile_ema}
\end{equation}
Thus, empty or undersupported probes do not reset an operation's scores. The proposer baseline is initialized as $b_0=0$. Retirement requires $U\leq0.05$ and $J\geq0.90$; reactivation occurs when $U>0.10$ or $J<0.80$.

\paragraph{Model initialization and data.} Qwen2.5-VL-7B-Instruct and Qwen3-VL-4B/8B-Instruct receive proposer-format SFT on 1,000 examples selected from 2,000 duality-operation code examples generated by DeepSeek-V4-Pro, followed directly by GRPO, without a Video-R1 task-SFT warm-up. The RL dataset contains 10K samples, comprising 5K from Video-R1 and 5K from SpatialLadder. Training uses 16 NVIDIA A100 GPUs, DeepSpeed ZeRO-3, 8 original rollouts and 4 dual rollouts when applicable, for one epoch.

\paragraph{Training-data scale.} For context, Video-R1 constructs a 260K-sample RL corpus, while SpatialLadder trains progressively on its 26,610-sample spatial corpus~\cite{Feng_2025_Video-R1,Li_2026_SpatialLadder}; SAGE uses 10K samples for RL.

\paragraph{Optimization details.} We use the framework-default AdamW optimizer with a cosine learning-rate schedule. The learning rate is $10^{-6}$, with weight decay $0.01$, maximum gradient norm $5$, per-device batch size $1$, and one gradient accumulation step. Maximum prompt and completion lengths are 16,384 and 768 tokens, respectively. Rollouts use temperature $1.0$ and top-$p=0.95$. Training uses bf16, gradient checkpointing, and Flash Attention~2, with a maximum visual input budget of 401,408 pixels. We set the KL coefficient to $\beta=0.04$ and the mastered-operation spot-check probability to $p_f=0.2$.

\paragraph{Evaluation protocol.} Literature-reported scores retain their source-specific inference settings, documented in Appendix~\ref{sec:baseline_sources}. The recorded in-house configuration specifies greedy decoding (temperature $\approx0$, top-p $=0.001$), at most 1,024 output tokens, and 16 frames per video.

\section{Additional Experimental Results}
\label{sec:additional_results}

This appendix provides controlled ablation protocols, paired-metric definitions, and additional analyses of training and operation-pool behavior.

\subsection{Controlled Ablation Protocols}
\label{app:ablation_protocol}
All controlled variants use the same backbone, proposer-format initialization, training data, and evaluation protocol unless the named component is changed. Average Score is the unweighted mean of the 13 full-benchmark scores. Paired diagnostics use a common evaluation set disjoint from the training probe buffer.

The ablation backbone is Qwen2.5-VL-7B-Instruct. Base and full-SAGE averages are calculated from the corresponding rows of Tables~\ref{tab:simplified_results} and~\ref{tab:spatial_comparison}, using their displayed benchmark scores and rounding only after aggregation. The top SAGE row is the complete configuration shared by the relevant panel-wise comparisons.

\paragraph{Duality supervision.}
Table~\ref{tab:duality_supervision} compares paired training on a fixed stream of valid original--dual examples with matched transformations, labels, and rollout budgets. Dual augmentation uses task and format rewards without a consistency bonus, updating first on the original input and then on the dual input. The ungated consistency variant removes the correctness gate, whereas SAGE retains it and combines both sides' gradients before one solver update. Thus comparison with dual augmentation changes both reward design and update scheduling; Appendix~\ref{app:joint_update_analysis} isolates the latter with the loss components held fixed. Under the SAGE reward, one-sided-correct pairs retain task supervision but receive no consistency bonus. Original-only GRPO is a reference, not an equal-rollout-cost control.

\paragraph{Operation scheduling and discovery.}
In Panel C of Table~\ref{tab:pool_ablation}, Fixed-all applies every applicable operation, while Random-$K$ randomly selects at most $K$. Adaptive-$K$ balances operation usage, prioritizing less-used operations to provide comparable training exposure. It does not use probe results or consistency scores to set priorities. Random-$K$ and Adaptive-$K$ share the active-capacity limit and solver rollout budget. Panel D removes one mechanism from SAGE at a time: removing discovery retains the initial operations and probe-based management; removing retirement disables Active-to-Mastered transitions; removing reactivation disables re-entry of mastered operations.

\paragraph{Operation families.}
Table~\ref{tab:ablation_refined} restricts admissible operations to visual or linguistic families under the same training recipe and rollout budget; their combination is represented by the top SAGE reference row.

\paragraph{Shared-model updates.}
Table~\ref{tab:shared_ablation} isolates parameter sharing and solver-priority projection. The unprojected variant retains the same proposer weight and objectives; the separate proposer--solver uses two trainable role models initialized from the same backbone. The table reports performance and resource use under matched rollout budgets.

\subsection{Paired-Metric Definitions}
\label{app:paired_metric_protocol}

\paragraph{Task scores versus binary correctness.}
For each of the $N$ evaluated pairs, let $s_i,s'_i\in[0,1]$ be the original and dual task scores returned by the benchmark scorer, and let $c_i$ indicate agreement under the declared answer map. The paired scores are
\begin{equation}
A_{\rm orig}=\frac{1}{N}\sum_i s_i,\qquad
A_{\rm dual}=\frac{1}{N}\sum_i s'_i,\qquad
C=\frac{1}{N}\sum_i c_i.
\end{equation}
Original-side accuracy is $A_{\rm orig}$, not the mean over both sides. For score-thresholded diagnostic profiles, $b_i=\mathbf1[s_i>0.5]$ and $b'_i=\mathbf1[s'_i>0.5]$; a score of exactly $0.5$ counts as wrong. Joint correctness $J=N^{\rm CC}/N$ is distinct from the average task score.

\paragraph{Correctness-conditioned consistency.}
Joint correctness and wrong-consistency are
\begin{equation}
J=\frac{1}{N}\sum_i b_i b'_i,\qquad
C_{\rm wrong}=\frac{1}{N}\sum_i(1-b_i)(1-b'_i)c_i.
\end{equation}
Under exact canonical correctness and a valid bijective answer map, both-correct pairs are consistent and one-sided-correct pairs are inconsistent, giving $C=J+C_{\rm wrong}$. Score-thresholded numerical diagnostics need not satisfy this identity: two numerical answers may both pass the task threshold without obeying the declared answer relation. Correct-consistency therefore measures $N^{-1}\sum_i b_i b'_i c_i$, not joint correctness. The proofs in Appendix~\ref{app:proof} use exact canonical correctness, not thresholded numerical scores.

\paragraph{Numerical scoring and consistency.}
The repository's VSI-Bench regression, SPAR numerical, and OST estimation scorers use ten relative-error thresholds spanning $0.50$ to $0.95$, giving per-item credit in increments of $0.1$. Their boundary conventions differ: video regression and OST estimation use strict inequalities, whereas SPAR uses a non-strict inequality; zero-target handling also differs. The paired evaluation implementation separately normalizes numerical answers to two decimal places before testing their mapped equality. This answer-equality test is not the fractional-credit task score. Recomputations therefore retain the original per-item scores and normalized mapped answers rather than replacing benchmark-specific scorers with a shared tolerance.

\paragraph{Aggregation and numerical precision.}
Paired diagnostics pool valid applicable pairs rather than averaging benchmark-level percentages. All aggregation precedes rounding to one decimal place. Full-benchmark scores retain their benchmark-specific aggregation, and Average Score averages the 13 completed benchmark scores equally. Missing benchmarks are not treated as zero or silently dropped.

\subsection{Further Analysis Details}
\label{app:further_analysis}

\paragraph{Checkpoint-level dynamics.}
Table~\ref{tab:checkpoint_ccww} uses the same metrics as Table~\ref{tab:duality_supervision}, with Base and Final aligned to its reference rows. Average Score averages the 13 benchmarks; the paired diagnostics distinguish joint correctness from wrong-consistent predictions.

\begin{table}[t]
\centering
\caption{Checkpoint-level evolution of correctness and duality consistency.}
\label{tab:checkpoint_ccww}
\small
\begin{tabular}{lcccc}
\toprule
Checkpoint & \makecell{Average\\Score} & Consistency & \makecell{Joint\\Correctness} & \makecell{Wrong-\\consistent} \\
\midrule
Base & 43.8 & 61.4 & 24.1 & 37.3 \\
Early  & 49.2 & 70.7 & 37.8 & 32.9 \\
Middle & 51.5 & 74.1 & 42.9 & 31.2 \\
Final & 52.9 & 75.9 & 45.7 & 30.2 \\
\bottomrule
\end{tabular}
\end{table}

\paragraph{Generalization to unseen operations.}
The held-out transformations are never activated during training and are evaluated without further model updates. Reporting both consistency and joint correctness distinguishes transferable reasoning from merely agreeing under unfamiliar transformations.

\begin{table}[t]
\centering
\caption{Generalization to unseen duality operations.}
\label{tab:unseen_ops}
\small
\begin{tabular}{lccccc}
\toprule
Method & Seen Cons. & Unseen Cons. & Seen J & Unseen J & $\Delta J$ \\
\midrule
Base & 61.67 & 56.64 & 24.00 & 22.31 & $-1.69$ \\
GRPO & 68.50 & 62.89 & 35.95 & 32.51 & $-3.44$ \\
SAGE & 76.25 & 73.02 & 45.62 & 43.11 & $-2.51$ \\
\bottomrule
\end{tabular}
\end{table}

\begin{table}[!t]
\centering
\caption{Quality of generated duality operations before pool admission.}
\label{tab:proposal_quality}
\small
\begin{tabular}{lccccc}
\toprule
Model / Stage & \makecell{Valid\\code (\%)} & Executable (\%) & Applicable (\%) & \makecell{Correct\\mapping (\%)} & Accepted (\%) \\
\midrule
Initial proposals & 91.67 & 81.67 & 65.00 & 53.33 & 16.67 \\
After training & 96.67 & 93.33 & 93.33 & 88.33 & 46.67 \\
\bottomrule
\end{tabular}
\end{table}

\paragraph{Proposal quality.}
Tables~\ref{tab:proposal_quality} and~\ref{tab:proposal_novelty} assess parsing and execution, applicability, answer-map correctness, and pool admission. Comparing generated proposals with the initial pool further distinguishes valid operations from novel and useful ones.

\begin{table}[t]
\centering
\caption{Novelty and quality of generated duality operations.}
\label{tab:proposal_novelty}
\small
\begin{tabular}{lcccc}
\toprule
Proposal source & Total & Valid (\%) & Novel (\%) & Useful (\%) \\
\midrule
Initial pool & 6 & 100.00 & -- & 100.00 \\
Generated proposals & 60 & 88.33 & 70.00 & 56.67 \\
\bottomrule
\end{tabular}
\end{table}

\paragraph{Reward weight and active-set size.}
Tables~\ref{tab:lambda} and~\ref{tab:k} vary $\lambda$ and $K$ to examine the balance between task performance, consistency, and operation coverage while keeping the remaining settings fixed.

\begin{table}[t]
\centering
\caption{Sensitivity to the duality reward weight $\lambda$.}
\label{tab:lambda}
\small
\begin{tabular}{lcccc}
\toprule
$\lambda$ & \makecell{Average\\Score} & Consistency & \makecell{Joint\\Correctness} & \makecell{Wrong-\\consistent} \\
\midrule
0.0 & 51.6 & 74.9 & 43.9 & 31.0 \\
0.1 & 52.5 & 75.6 & 45.0 & 30.6 \\
0.3 & 52.9 & 75.9 & 45.7 & 30.2 \\
0.5 & 52.8 & 76.3 & 45.6 & 30.7 \\
0.7 & 52.2 & 76.8 & 45.0 & 31.8 \\
1.0 & 51.5 & 77.0 & 44.0 & 33.0 \\
\bottomrule
\end{tabular}
\end{table}

\begin{table}[!t]
\centering
\caption{Sensitivity to the number of active duality operations.}
\label{tab:k}
\small
\begin{tabular}{lcccc}
\toprule
$K$ & \makecell{Average\\Score} & Consistency & \makecell{Joint\\Correctness} & \makecell{Wrong-\\consistent} \\
\midrule
1 & 51.6 & 73.6 & 43.5 & 30.1 \\
2 & 52.6 & 75.3 & 45.2 & 30.1 \\
3 & 52.9 & 75.9 & 45.7 & 30.2 \\
4 & 53.0 & 76.0 & 45.9 & 30.1 \\
6 & 52.7 & 75.7 & 45.4 & 30.3 \\
\bottomrule
\end{tabular}
\end{table}

\paragraph{Limitations and discussion.}
\label{sec:limitations}
Consistency does not imply correctness: two wrong predictions can still obey the answer mapping. Correctness gating reduces this training incentive but does not eliminate wrong-consistent predictions at evaluation. Executable checks cannot guarantee semantic validity for all ambiguous operations, particularly color, orientation, viewpoint, and negation transformations. SAGE also incurs additional rollout computation and remains dependent on backbone capability. Results are primarily based on single runs; multiple-seed experiments and broader backbone comparisons remain necessary.

\subsection{Additional Video Baselines}

\begin{table}[!t]
\centering
\caption{Additional third-party video results omitted from the compact main table. Values are quoted from the cited papers.}
\label{tab:additional_video_baselines}
\small
\setlength{\tabcolsep}{2.5pt}
\resizebox{\textwidth}{!}{%
\begin{tabular}{lcccccc}
\toprule
Model & VSI-Bench & VideoMMMU & MMVU & MVBench & TempCmps & VideoMME \\
\midrule
LLaMA-VID \cite{Li_2024_LLaMA-VID} & - & - & - & 41.9 & 45.6 & - \\
VideoLLaMA2 \cite{Cheng_2024_VideoLLaMA2} & - & - & 44.8 & 54.6 & - & 47.9 \\
LongVA-7B \cite{Zhang_2025_LongVA-7B} & 29.2 & 23.9 & - & - & 56.9 & 52.6 \\
VILA-1.5-40B \cite{Lin_2024_VILA} & 31.2 & 34.0 & - & - & - & 60.1 \\
Video-UTR-7B \cite{Yu_2025_Video-UTR} & - & - & - & 58.8 & 59.7 & 52.6 \\
LLaVA-OneVision-7B \cite{Li_2025_LLaVA-OneVision} & 32.4 & 33.8 & 49.2 & 56.7 & - & 58.2 \\
Kangaroo-8B \cite{Liu_2026_Kangaroo} & - & - & - & 61.1 & 62.5 & 56.0 \\
\bottomrule
\end{tabular}
}
\end{table}

Table~\ref{tab:additional_video_baselines} extends the comparison to earlier and general-purpose video VLMs whose reported coverage is incomplete across the six benchmarks. Their strongest results are concentrated on individual benchmarks, so we retain these sparsely populated rows as contextual references rather than use them for aggregate comparison.

\subsection{Training Dynamics}
\label{sec:training_dynamics}

\begin{figure}[t]
\centering
\includegraphics[width=\linewidth]{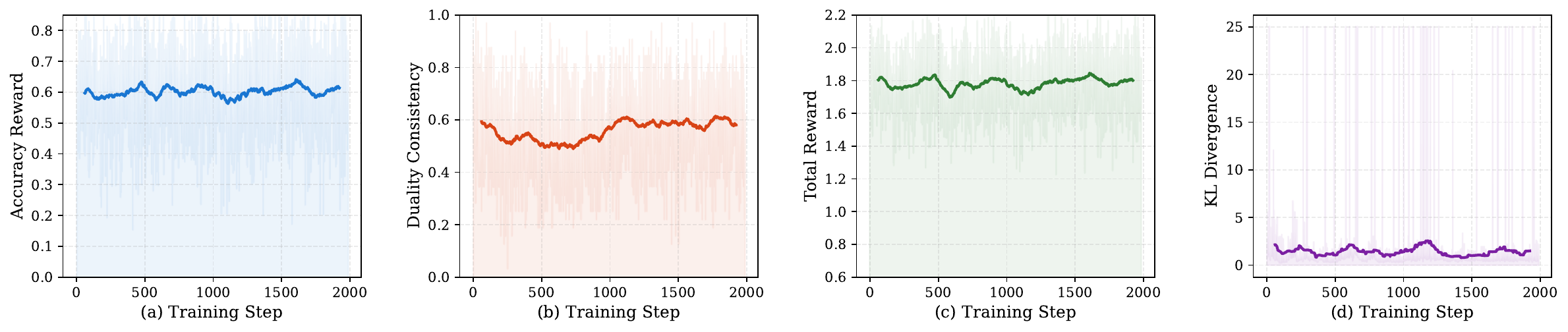}
\caption{Training dynamics of SAGE. Accuracy reward remains stable, duality consistency recovers after an initial change, total reward remains stable, and KL divergence stays bounded.}
\label{fig:training_curves}
\end{figure}

Figure~\ref{fig:training_curves} shows an early decline in consistency followed by recovery, while accuracy and total rewards remain stable. The bounded KL divergence is consistent with stable policy adaptation over the observed run.

\section{Properties of Answer-Level Equivariance}
\label{app:proof}

We state only the properties needed by SAGE. We do not claim a VC-dimension or generalization bound for an overparameterized VLM.

\subsection{Setup}

We use the canonical decoder $\hat a_\theta$ from Section~\ref{sec:duality_probing} and deterministic ground truth $a^*(x)$. This analysis uses exact canonical 0--1 correctness; task-specific tolerant reward verification is outside the theorem. For a valid $\mathcal T=(T_{\mathcal X},T_{\mathcal A},\mathcal S)$ with $\Pr_{x\sim\mathcal D}[x\in\mathcal S]>0$, all probabilities below are over $x\sim\mathcal D_{\mathcal S}$, and extraction failure $\bot$ counts as an incorrect answer. Define
\begin{equation}
\begin{aligned}
 R(\theta)&=\Pr\!\left[\hat a_\theta(x)\neq a^*(x)\right],\\
 R^{\mathcal T}(\theta)&=\Pr\!\left[\hat a_\theta(T_{\mathcal X}(x))\neq T_{\mathcal A}(a^*(x))\right],\\
 R_{\mathrm{pair}}(\theta)&=\tfrac12\left(R(\theta)+R^{\mathcal T}(\theta)\right).
\end{aligned}
\label{eq:pair_risk}
\end{equation}

\subsection{Effect of Group-Relative Normalization}
\label{sec:reward_normalization}

The correctness gate makes the consistency term a cross-side reward reweighting rather than an additional source of labels. Let $x_i\in\{0,1\}$ and $y_j\in\{0,1\}$ denote exact correctness of the original and dual completions. By bijectivity and Eq.~\eqref{eq:duality_axiom}, $g_{ij}c_{ij}=x_i y_j$, and hence
\begin{equation}
 r_{\rm cons}(o_i)=\lambda x_i\bar y,
 \qquad
 \bar y=\frac{2}{G}\sum_{j=1}^{G/2}y_j,
 \label{eq:cross_side_reweighting}
\end{equation}
with the symmetric identity $r'_{\rm cons}(o'_j)=\lambda y_j\bar x$ on the dual side, where $\bar x=G^{-1}\sum_{i=1}^{G}x_i$.

For a reward vector $z$ and $\epsilon\geq0$ with $\sigma_z+\epsilon>0$, define its group-relative normalization by
\begin{equation}
 \mathsf Z_\epsilon(z)_i
 =\frac{z_i-\bar z}{\sigma_z+\epsilon}.
 \label{eq:reward_standardization_operator}
\end{equation}

\begin{proposition}[Affine-invariant boundary case]
\label{prop:reward_affine_boundary}
Let the original-side reward before the consistency term be $q_i=a x_i+c$, where $a>0$ and $c$ is constant within the rollout group. Then adding Eq.~\eqref{eq:cross_side_reweighting} produces a positive affine transformation of $q$. If $\sigma_q>0$, then $\mathsf Z_0(q+r_{\rm cons})=\mathsf Z_0(q)$; for $\epsilon>0$, the two normalized vectors are positive scalar multiples of one another. If $\sigma_q=0$ and $\epsilon>0$, both normalized vectors are zero. Outside this affine condition, group normalization provides no such invariance guarantee.
\end{proposition}

\begin{proof}
Set $k=\lambda\bar y\geq0$. Since $q_i=a x_i+c$,
\begin{equation}
 q_i+kx_i
 =\left(1+\frac{k}{a}\right)q_i-\frac{kc}{a}
 =Aq_i+B,
\end{equation}
where $A=1+k/a>0$ and $B=-kc/a$. Thus $\overline{Aq+B}=A\bar q+B$ and $\sigma_{Aq+B}=A\sigma_q$. Substitution into Eq.~\eqref{eq:reward_standardization_operator} proves exact invariance for $\epsilon=0$ and positive collinearity for $\epsilon>0$. When $\sigma_q=0$, $x_i$ is constant because $a>0$, so the augmented rewards are also constant. If the base reward is not affine in binary correctness, adding $kx$ is not generally a positive affine transformation of the reward vector, and the preceding invariance argument does not apply.
\end{proof}

Binary accuracy with a group-constant format reward is the principal boundary case covered by Proposition~\ref{prop:reward_affine_boundary}. Graded task rewards or completion-dependent format rewards can violate its affine premise, allowing cross-side reweighting to change relative advantages. The proposition therefore identifies when the bonus is redundant under normalization; it does not attribute gains from paired task supervision, joint updating, or adaptive operation selection to the bonus alone.

\subsection{Pair-Risk Lower Bound}

\begin{theorem}[Pair-risk lower bound]
\label{thm:suboptimality}
For any valid duality operation,
\begin{enumerate}
    \item[(i)] $R_{\mathrm{pair}}(\theta)\geq(1-\mathcal C_{\mathcal T}(\theta))/2$;
    \item[(ii)] if $R(\theta)=0$, then $R^{\mathcal T}(\theta)=1-\mathcal C_{\mathcal T}(\theta)$.
\end{enumerate}
\end{theorem}

\begin{proof}
Fix $x\in\mathcal S$ and write $a=\hat a_\theta(x)$ and $a'=\hat a_\theta(T_{\mathcal X}(x))$. Let
\begin{equation}
 I_{\rm cons}(x)=\mathbf 1\!\left[a\in\mathcal A\ \land\ a'\in\mathcal A\ \land\ T_{\mathcal A}(a)=a'\right].
\end{equation}
If both predictions are correct, then $a=a^*(x)$ and, by Eq.~\eqref{eq:duality_axiom}, $a'=T_{\mathcal A}(a^*(x))=T_{\mathcal A}(a)$, so $I_{\rm cons}(x)=1$. Taking the contrapositive, every inconsistent sample has at least one incorrect side, including any sample with an extraction failure. Hence
\begin{equation}
    \mathbf 1[a\neq a^*(x)]+\mathbf 1[a'\neq T_{\mathcal A}(a^*(x))]
    \geq 1-I_{\rm cons}(x).
\end{equation}
Taking expectations and dividing by two proves (i). If $R(\theta)=0$, then $a=a^*(x)$ almost surely. On those samples, $I_{\rm cons}(x)=1$ if and only if $a'=T_{\mathcal A}(a^*(x))$; therefore $R^{\mathcal T}(\theta)=1-\mathcal C_{\mathcal T}(\theta)$, proving (ii). Perfect consistency alone does not imply zero pair risk: a consistent model may be wrong on both sides. This is precisely why Eq.~\eqref{eq:consistency_reward} is correctness-gated.
\end{proof}

\subsection{Noise Robustness of Hysteretic Pool Updates}
\label{sec:hysteresis_stability}

\begin{proposition}
\label{prop:hysteresis}
Let $U_t^\star,J_t^\star$ denote the noiseless operation scores and let the transition statistics be $\widehat U_t=U_t^\star+\epsilon^U_t$ and $\widehat J_t=J_t^\star+\epsilon^J_t$, where $|\epsilon^U_t|\leq\delta_U$ and $|\epsilon^J_t|\leq\delta_J$. Suppose Eqs.~\eqref{eq:retire}--\eqref{eq:forget} are applied to $\widehat U_t,\widehat J_t$ and
\begin{equation}
 \tau_{\mathrm{hi}}-\tau_{\mathrm{lo}}>2\delta_U,
 \qquad
 \kappa_{\mathrm{hi}}-\kappa_{\mathrm{lo}}>2\delta_J.
\end{equation}
After an Active-to-Mastered transition at refresh $t_1$, a later Mastered-to-Candidate transition at $t_2$ requires at least one of
\begin{align}
 U_{t_2}^\star-U_{t_1}^\star&>\tau_{\mathrm{hi}}-\tau_{\mathrm{lo}}-2\delta_U,\label{eq:hysteresis_u}\\
 J_{t_1}^\star-J_{t_2}^\star&>\kappa_{\mathrm{hi}}-\kappa_{\mathrm{lo}}-2\delta_J.\label{eq:hysteresis_j}
\end{align}
Thus bounded measurement noise alone cannot cause a reversal when the underlying scores are unchanged.
\end{proposition}

\begin{proof}
Mastery at $t_1$ implies $U_{t_1}^\star\leq\tau_{\mathrm{lo}}+\delta_U$ and $J_{t_1}^\star\geq\kappa_{\mathrm{hi}}-\delta_J$. Reactivation occurs through one of the two conditions in Eq.~\eqref{eq:forget}. If $\widehat U_{t_2}>\tau_{\mathrm{hi}}$, then $U_{t_2}^\star>\tau_{\mathrm{hi}}-\delta_U$; subtracting the bound at $t_1$ gives Eq.~\eqref{eq:hysteresis_u}. If $\widehat J_{t_2}<\kappa_{\mathrm{lo}}$, then $J_{t_2}^\star<\kappa_{\mathrm{lo}}+\delta_J$; subtracting from the bound at $t_1$ gives Eq.~\eqref{eq:hysteresis_j}. Both lower bounds are positive by assumption.
\end{proof}

\subsection{Joint Versus Sequential Paired Updates}
\label{app:joint_update_analysis}

\paragraph{Comparison and potential advantage.}
SAGE evaluates the original and dual losses at the same parameter state and combines their gradients before a solver update. The sequential dual-augmentation comparator instead updates on the original input and then on the dual input at the intermediate parameters. Joint updating avoids this within-pair parameter drift and uses one rather than two optimizer steps. The following analysis isolates update scheduling: both methods optimize the same two loss components. It does not change SAGE's reward or replace it with a single all-or-nothing pair reward.

Let $\ell_o(\theta)$ and $\ell_d(\theta)$ be fixed differentiable original and dual loss components, and let $L(\theta)=\tfrac12(\ell_o(\theta)+\ell_d(\theta))$. Match first-order gradient weights by comparing one joint step of size $\alpha$ with two sequential steps of size $\alpha/2$:
\begin{align}
\theta^+_{\rm joint}
&=\theta-\frac{\alpha}{2}
\bigl(\nabla\ell_o(\theta)+\nabla\ell_d(\theta)\bigr),
\label{eq:paired_joint_step}\\
\theta_{1/2}
&=\theta-\frac{\alpha}{2}\nabla\ell_o(\theta),\qquad
\theta^+_{\rm seq}
=\theta_{1/2}-\frac{\alpha}{2}\nabla\ell_d(\theta_{1/2}).
\label{eq:paired_sequential_step}
\end{align}
For a locally Lipschitz Hessian of $\ell_d$, Taylor expansion as $\alpha\to0$ yields
\begin{equation}
\theta^+_{\rm seq}-\theta^+_{\rm joint}
=\frac{\alpha^2}{4}
\nabla^2\ell_d(\theta)\nabla\ell_o(\theta)+O(\alpha^3).
\label{eq:paired_order_error}
\end{equation}
Thus, sequential updating introduces an order-dependent second-order term absent from joint updating. Its sign is not universally unfavorable; the proposition below gives explicit conditions for a strict joint-update advantage.

\begin{proposition}[Eliminating sequential residual error under conflicting objectives]
\label{prop:paired_joint_advantage}
Let $\theta,a\in\mathbb R^d$, $a\neq0$, and consider the equal-curvature quadratic losses
\[
\ell_o(\theta)=\tfrac12\|\theta-a\|_2^2,\qquad
\ell_d(\theta)=\tfrac12\|\theta+a\|_2^2.
\]
With exact gradients and constant $0<\alpha<2$, repeated joint updates in Eq.~\eqref{eq:paired_joint_step} converge from any finite initialization to the minimizer $\theta^*=0$ of $L$. Original-then-dual sequential updates in Eq.~\eqref{eq:paired_sequential_step}, observed after each completed pair, converge to $-\alpha a/(4-\alpha)$. Their limiting excess average loss is
\begin{equation}
\lim_{t\to\infty}
\left[L(\theta_t^{\rm seq})-L(\theta_t^{\rm joint})\right]
=\frac{\alpha^2\|a\|_2^2}{2(4-\alpha)^2}>0.
\label{eq:paired_residual_gap}
\end{equation}
\end{proposition}

\begin{proof}
The average loss is $L(\theta)=\tfrac12\|\theta\|_2^2+\tfrac12\|a\|_2^2$. A joint step gives $\theta_{t+1}^{\rm joint}=(1-\alpha)\theta_t^{\rm joint}$, which converges to zero because $|1-\alpha|<1$. A sequential cycle gives
\begin{equation}
\theta_{t+1}^{\rm seq}
=\left(1-\frac{\alpha}{2}\right)^2\theta_t^{\rm seq}
-\frac{\alpha^2a}{4}.
\end{equation}
Its contraction factor lies in $(0,1)$, and its unique fixed point is $-\alpha a/(4-\alpha)$. Substituting this point and zero into $L$ proves Eq.~\eqref{eq:paired_residual_gap}. Reversing the update order changes the sign of the fixed point but leaves the limiting excess loss unchanged.
\end{proof}

\paragraph{Interpretation and scope.}
The two objectives prefer different parameters, representing conflict between original and dual updates. Joint gradients balance these preferences at one parameter state, whereas cyclic sequential updates retain an order-dependent bias under constant steps. The proposition establishes a strict asymptotic loss advantage in this quadratic model, not superiority at every finite iteration or a guarantee of benchmark accuracy for a nonconvex VLM. Finite rollouts, GRPO clipping, momentum, AdamW preconditioning, and changing training samples are outside the proof. This is an advantage of update scheduling; it does not prove an independent benefit of the consistency reward. Dual augmentation that already accumulates both gradients before one optimizer step is not the sequential comparator.

\paragraph{Optimizer overhead and speed.}
Let $C_{\rm pair}$ be the common cost of generating and scoring both sides and computing their gradients, and let $U_{\rm opt}>0$ be the cost of one optimizer step, including its associated synchronization when applicable. Under this idealized decomposition,
\begin{equation}
T_{\rm joint}=C_{\rm pair}+U_{\rm opt},\qquad
T_{\rm seq}=C_{\rm pair}+2U_{\rm opt}.
\label{eq:paired_update_cost}
\end{equation}
Joint updating saves one optimizer step per processed source pair. The modeled speedup is $T_{\rm seq}/T_{\rm joint}=1+U_{\rm opt}/(C_{\rm pair}+U_{\rm opt})$, approaching one when generation and gradient computation dominate. Neither side's forward/backward work disappears, and communication overlap may differ in practice; a twofold end-to-end speedup is not implied.

\paragraph{Matched comparison protocol.}
The existing dual-augmentation and SAGE rows compare performance and resource use. A matched speed comparison uses identical hardware, paired inputs, rollout budgets, and timing boundaries. Full SAGE also incurs online proposal and probing overhead, so saving an optimizer step does not guarantee lower end-to-end training time. The theorem does not predetermine benchmark accuracy or measured runtime.

For the first-order-matched sequential control, each step uses half the joint learning rate; SAGE's existing learning-rate setting is unchanged. The schedule is indexed by processed source pairs, and per-pair regularization weights are matched rather than applying the full KL penalty twice. Both sides of a pair are generated before either update, retaining the same rollout budget. Rewards, advantages, and behavior-policy likelihoods are fixed within that pair, while the second sequential gradient is evaluated at the intermediate parameters. The proof holds losses fixed; comparisons that also change consistency rewards or online operation discovery measure their combined effect and cannot attribute the entire empirical gap to scheduling alone.

\section{Copyrights}

We acknowledge that the image used in Figure~\ref{fig:intro} is credited to Kamran Aydinov on Magnific. The figure is used solely for academic illustration purposes.

\section{Case Study: Duality Consistency in Practice}

\paragraph{Visual operations.} Figure~\ref{fig:case_visual} presents candidate diagnostic pairs from OmniSpatial. Rows~1--3 ask for quantities whose ground truth is preserved by the stated transformation, but the model's predictions change: a person's estimated speed changes from 0.43\,m/s to 5.7\,m/s after horizontal flip and to 24.7\,m/s after $180^\circ$ rotation; another speed changes from 3.5\,m/s to 0.4\,m/s after vertical flip; and a car-speed estimate changes from 2.33\,m/s to 1.28\,m/s after a flip plus color inversion. Row~4 is a boundary counterexample, not a valid consistency failure: color inversion changes the category-defining cue in ``big black bear,'' so the applicability predicate must reject this pair.

\paragraph{Linguistic operations.} Figure~\ref{fig:case_linguistic} shows failures under two linguistic operations. Under \emph{option permutation}, the order of multiple-choice options is reversed while keeping the question and image unchanged. The model selects ``B'' in both cases, but ``B'' maps to ``right'' (correct) in the original and ``left'' (wrong) after permutation, indicating sensitivity to option position rather than consistent tracking of answer content. Under \emph{question negation}, ``Which hand does the woman use?'' becomes ``Which hand does NOT the woman use?'' yet the model returns the identical answer, demonstrating a failure to track logical polarity in this example.

The valid pairs motivate the SAGE reward, while the rejected black-bear pair illustrates why every operation needs an applicability domain. The method applies consistency supervision only when the answer transformation is defined and the transformed sample remains semantically answerable.

\begin{figure}[t]
\centering
\includegraphics[width=\linewidth]{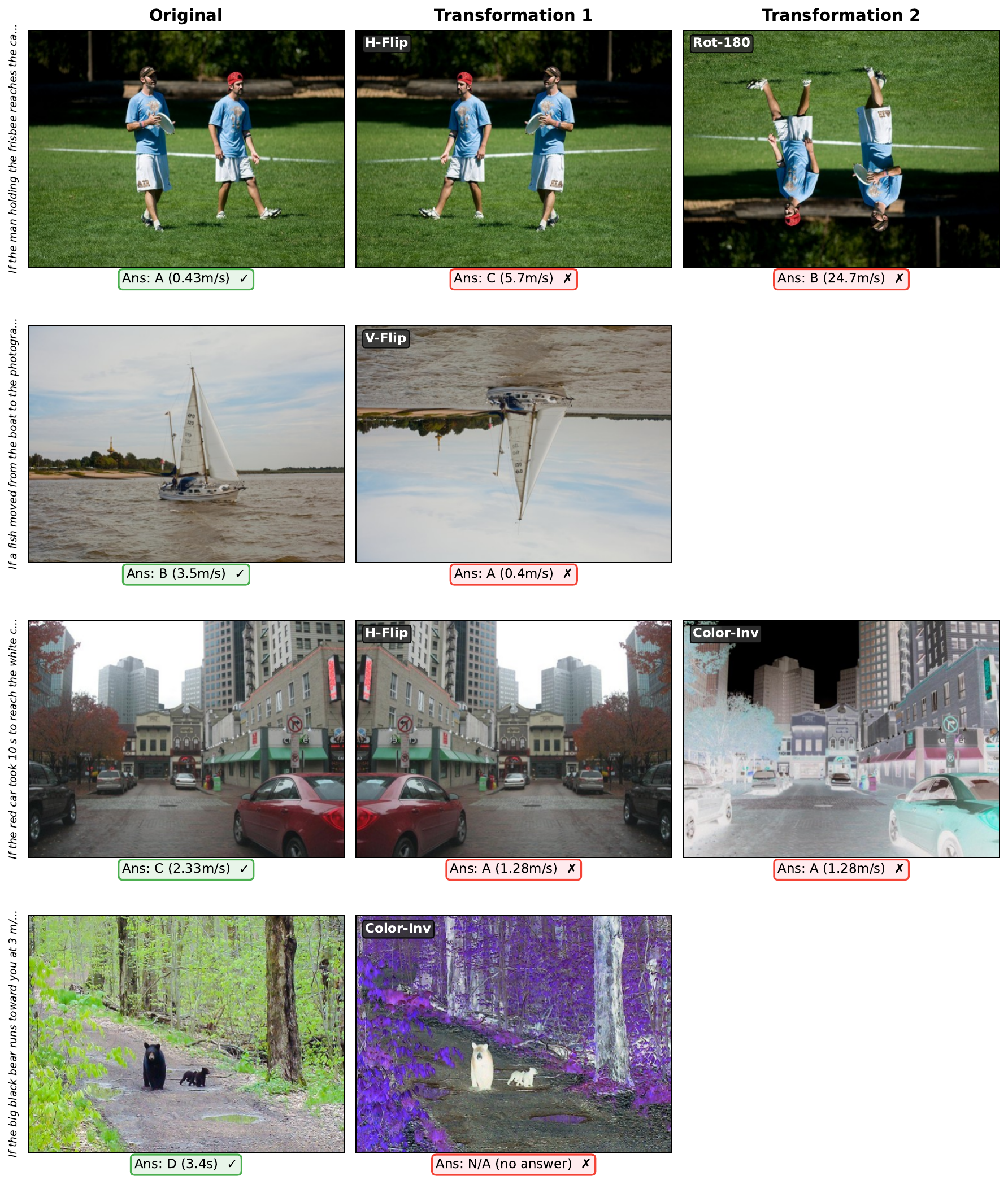}
\caption{Visual operation diagnostics. Rows~1--3 are answer-preserving pairs with inconsistent predictions (red). Row~4 is deliberately marked invalid: color inversion changes the category-defining ``black'' cue, so SAGE's applicability predicate excludes it from consistency supervision.}
\label{fig:case_visual}
\end{figure}

\begin{figure}[t]
\centering
\includegraphics[width=\linewidth]{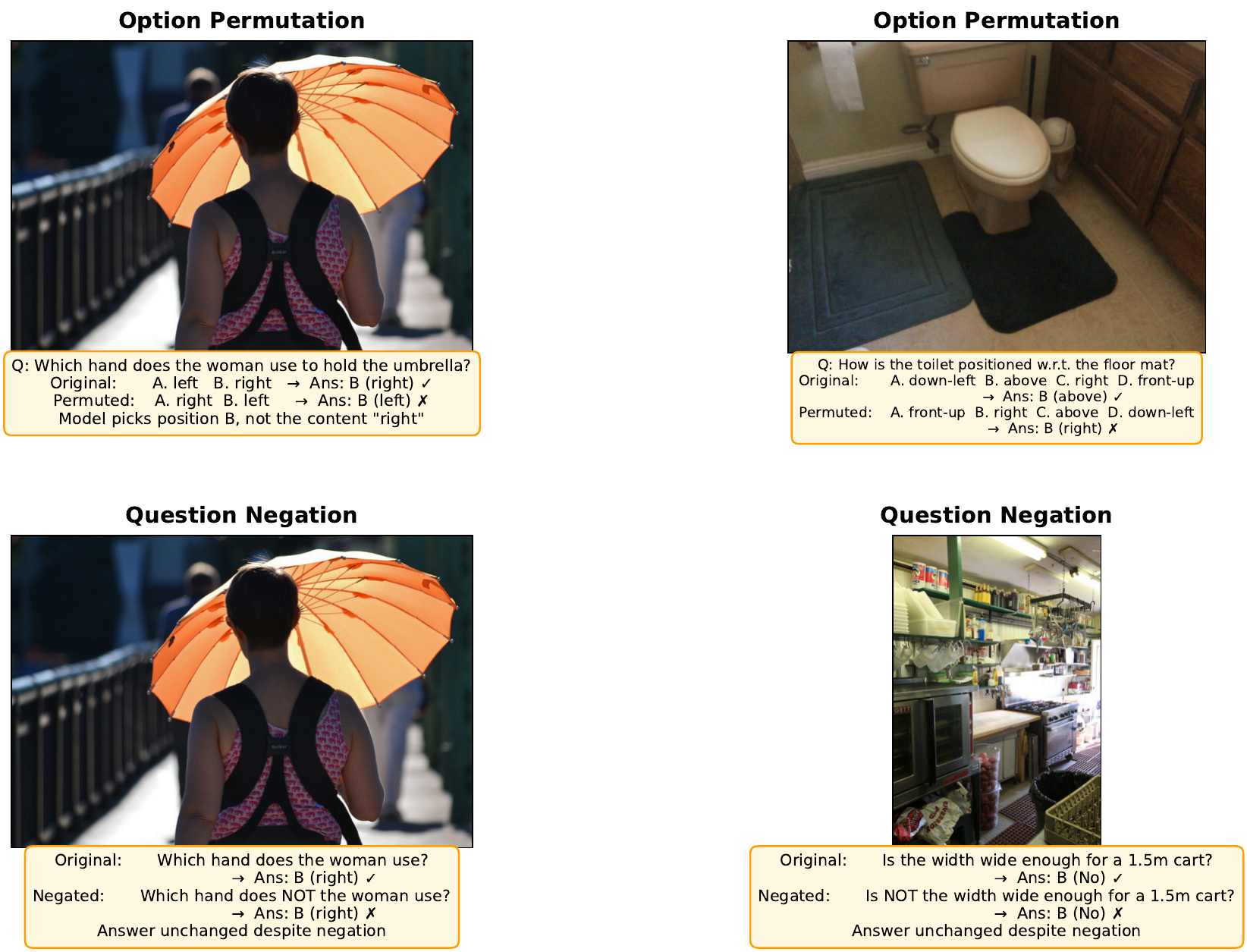}
\caption{Linguistic duality operation failure cases. \textbf{Top}: option permutation reverses the order of choices; the model selects the same position (B) rather than tracking the content. \textbf{Bottom}: question negation inserts ``NOT''; the model returns identical answers, ignoring the logical reversal.}
\label{fig:case_linguistic}
\end{figure}

\end{document}